\documentclass[11pt]{article}
\usepackage[margin=1in]{geometry}
\usepackage[T1]{fontenc}
\usepackage{lmodern}
\usepackage{graphicx}
\usepackage{booktabs}
\usepackage{amsmath,amssymb,amsthm}
\usepackage[numbers,sort&compress]{natbib}
\usepackage[colorlinks=true,linkcolor=blue,citecolor=blue,urlcolor=blue]{hyperref}
\usepackage{microtype}
\usepackage{placeins}

\newtheorem{theorem}{Theorem}
\newtheorem{proposition}{Proposition}
\newtheorem{corollary}{Corollary}
\newtheorem{lemma}{Lemma}
\newtheorem{assumption}{Assumption}
\theoremstyle{definition}
\newtheorem{definition}{Definition}
\theoremstyle{remark}
\newtheorem{remark}{Remark}

\title{The Kinetics of Training: A Driven-Nucleation Rate Law\\ for Emergence, Plasticity Loss, and Circuit Control in Language Models}
\author{Lei Dong\\ \small Independent Researcher \\ \small \texttt{dongleiit@hotmail.com}}
\date{July 2026}

\newcommand{\tstar}{t^{*}}
\newcommand{\eps}{\varepsilon}

\begin{document}
\maketitle

\begin{abstract}
A capability appears in a language model when the last parts of its circuit align in one stochastic attempt, and getting all but one right is worth nothing. We show this no-partial-credit joint alignment is the \textbf{rate-limiting step} of capability formation. Two fingerprints: in a shortcut-free apparatus a five-part circuit missing three waits as long as a three-part circuit missing three (1.19--1.37), so the wait counts \emph{missing} parts, not size; and on Pythia across seven capabilities and three scales, ablating one part leaves a median 17\% of the capability in 32 of 32 discriminating cells, where partial credit predicts 50--83\% ($p=2\times10^{-10}$), while a random non-part head leaves 100\%.

One rare event whose barrier grows with missing parts yields one rate equation --- sites $\times$ attempts $\times$ drive $\times\,e^{-\beta K}$, minus destruction --- read three ways, each preregistered with frozen constants. \emph{Forward}: a capability flat at baseline ignites at a step of our choosing once the mix passes a concentration floor (10/10 above, 0/12 below), and while still flat its arrival is datable from its precursor to 5\% median error on six held-out models. \emph{Backward}: the delay to learn a withheld capability grows with waiting until, past a critical step, it never ignites --- yet validation loss falls smoothly throughout, so standard monitors are blind to it. We locate the damage (heads commit to the base data) and isolate the cure: re-initializing only the query--key slices restores learnability (6/6) while the value slices do nothing (0/6). We prove the mechanism in a controlled gated-attention model, where occupation forces a finite deadline whose consequences need \emph{no mixing assumption}. \emph{Completed}: SGD's noise fails the fluctuation--dissipation test, so we install one and anneal, melt and pin circuits on schedule. Frozen factors predict never-run combinations to 1--8\%. Scope: conjunction circuits in transformers to 1.4B, two families.
\end{abstract}

\section{Introduction: training has a statics; it is missing its kinetics}
\label{sec:intro}

A substantial body of theory now describes learning in \emph{equilibrium} terms. Singular learning theory characterizes training stages as phase transitions of a Bayesian free energy \citep{watanabe2009algebraic,hoogland2024developmental,chen2023dynamical}; data-mixing thresholds have been derived as capacity-allocation optima \citep{gu2026mixing}; grokking has been modeled as a first-order transition \citep{rubin2024grokking} and as thermal escape from a metastable state \citep{ersoy2026noise}. This is, in the language of materials science, the \emph{equilibrium phase diagram} of training: which phases exist and which is preferred.

A phase diagram, however, answers none of the questions a practitioner asks. \emph{When} will this capability appear? \emph{How fast} will it grow, and can it be accelerated? \emph{Until when} can it still be learned --- and why does a network that has trained too long on the wrong data lose the ability to learn it at all \citep{dohare2024loss}? \emph{Which} circuit will a given perturbation destroy first? These are questions about \emph{rates}, and a free energy is not a rate. Metallurgy needed a second diagram --- the time--temperature--transformation (TTT) diagram --- before phase knowledge became process control. Training does not yet have one.

This paper constructs that missing kinetics. Its single object of study is the \emph{nucleation rate} of a computational circuit,
\begin{equation}
J \;=\; \underbrace{N}_{\text{sites}} \times \underbrace{\nu}_{\text{attempt freq.}} \times \underbrace{\sigma(c)}_{\text{supersaturation}} \times \underbrace{e^{-\beta K}}_{\text{barrier}} \;-\; \underbrace{D}_{\text{destruction}},
\label{eq:master}
\end{equation}
derived once from the stochastic dynamics of SGD (\S\ref{sec:theory}--\S\ref{sec:mathcore}) and then read three ways. Read \emph{forward} --- raising the drive $\sigma$ past a floor --- it is capability \emph{emergence}, with a derived clock that dates arrival before the capability is measurable (\S\ref{sec:emergence}). Read \emph{backward} --- as the site term $N$ is consumed by training itself --- it is \emph{loss of plasticity}, supplied with the quantitative law, horizon, mechanism, and repair that the plasticity literature identifies as missing (\S\ref{sec:plasticity}). And \emph{completed} --- injecting the temperature that SGD provably lacks --- it becomes \emph{process control}: annealing, selective dissolution, and quench-and-release schedules for individual circuits (\S\ref{sec:control}). Every factor in Eq.~\eqref{eq:master} is measured independently, frozen, and tested on never-run combinations (\S\ref{sec:synthesis}); every section of this paper either measures one factor or reads the equation in one direction.

Such a construction earns its keep only if it pays out on problems the field already cares about, in the field's own language. It pays out on three. \textbf{Loss of plasticity} is a recognized open problem: the phenomenon is established in \emph{Nature} \citep{dohare2024loss}, the warm-starting puzzle \citep{ash2020warm} stands, and the causes literature is explicitly partial --- intervening on any single proposed mechanism is insufficient \citep{lyle2024disentangling}. What exists are remedies without a law, a timetable, or a located mechanism; \S\ref{sec:plasticity} supplies each, quantitatively. \textbf{Capability-level data scheduling}: data-mixing methods \citep{xie2023doremi,liu2024regmix,ye2024mixing} optimize loss, not which capabilities form; the floor, the clock, and the shelf life answer ``for this mix and this size, will capability X emerge, when, and what is the deadline to still get it.'' \textbf{Emergence early-warning}: frontier safety frameworks poll evaluations and react \citep{ganguli2022predictability}; the precursor readout dates arrival while the capability itself is flat at baseline.

\textbf{Scope, stated up front.} Two model families (Pythia \citep{biderman2023pythia}, OLMo \citep{groeneveld2024olmo} --- the public suites with dense early checkpoints), scales to 1.4B, mechanism-level capabilities plus one task-level demonstration, causal constants from a controlled continued-training regime whose \emph{structure} (not constants) is the claim --- with the central causal results additionally reproduced from random initialization; and one explicitly unproved mathematical premise (Assumption~\ref{as:mixing}) that is flagged wherever it is used.

\section{The rate equation, forced step by step}
\label{sec:theory}

We separate throughout what is \emph{cited} (established results we do not re-prove), what is \emph{proved here} (under stated assumptions), and what is \emph{measured}. The derivation chain is short and each step forces the next.

\paragraph{Step 1 (cited): minibatch SGD, in the diffusion regime, is an SDE.}
With minibatches, the update is the full gradient plus zero-mean noise; in the diffusion-approximation regime (small steps, locally Gaussian batch noise) SGD is approximated by
$d\theta = -\nabla L(\theta)\,dt + \sqrt{2\eps\,\Sigma(\theta)}\,dW$
with explicit weak-error bounds \citep{li2017stochastic,mandt2017stochastic}. Production optimizers (Adam, clipping, schedules) deform this SDE without removing its two load-bearing features --- drift plus state-dependent noise --- but we make no stronger claim: the regime condition is part of the theory's stated scope (\S\ref{sec:scope}), and all causal constants are measured under plain SGD/AdamW where the approximation is controlled.

\paragraph{Step 2: why the barrier is exponential in $K$ --- two independent foundations.}
\emph{(A) Computational --- proven complexity theory for the tested task class.} For the family our premise tests instantiate --- $K$-wise conjunctions none of whose strict subsets carries signal, a property \emph{proved exactly} for our apparatus in Lemma~\ref{lem:nopartial} --- the exponential cost in $K$ is a theorem of learning theory, not an ansatz. Any statistical-query algorithm solving the $(n, k)$-sparse parity problem with $T$ queries of tolerance $\tau$ must satisfy $T/\tau^2 \ge \Omega(n^k)$ \citep{kearns1998efficient,blum1994weakly}; gradient methods with bounded gradient precision fall within the SQ framework (see the reduction discussed in \citealp[App.~A.1]{barak2022hidden} and references therein), and SGD provably operates near this limit, learning $k$-parities in $n^{O(k)}$ steps in the architectures where the guarantee is established \citep{barak2022hidden}. Complementing the cost bound, the staircase and leap theorems \citep{abbe2022merged,abbe2023leap} prove the dichotomy on which our rate-determining-step analysis rests: targets whose monomials each add at most one new coordinate (subset signal at every stage) are SGD-learnable in $O(d)$ samples --- an exact iff-characterization in the mean-field regime --- while a leap of size $\ell$ costs $\tilde\Theta(d^{\max(\ell-1, 1)})$ online-SGD steps where proven (Gaussian nested-monomial targets under layer-wise training; CSQ lower bounds in general; conjectured for vanilla SGD, with the leap-2 case established for standard minibatch SGD by \citealp{glasgow2024sgd}). These theorems govern the task class at large ambient dimension and in stated training regimes; our apparatus is a small-dimension member of the class whose defining no-subset-signal property holds exactly (Lemma~\ref{lem:nopartial}), and training is single-pass throughout (batch reuse can evade the leap barrier \citep{dandi2024benefit}).

\emph{What the theorems fix, and what they do not.} They fix the \emph{form} --- cost exponential in the number of jointly-required parts --- not the \emph{constant}. In the $n^{k}$ scaling, $n$ counts candidate coordinates over which the relevant support must be found; it is not a vocabulary size, and no reading of the bound assigns $\beta$ a universal value. Written as $e^{\beta}$ per part, $\beta$ is therefore a logarithm of an \emph{effective} candidate pool, to be measured per system exactly as an activation energy is measured, not derived, in chemical kinetics. The measured values are informative on precisely this point: in the toy the per-part cost is $\hat\beta = 1.75$ (summed family) and $\ge 2.40$ (parity family), i.e.\ effective pools of $\approx6$ and $\ge11$ --- the scale of the local window the design makes relevant, and far below the full context, showing the search is structured rather than blind. On real models the clock's $\beta \approx 0.23$ corresponds to an effective pool near unity: real circuit formation is not a blind combinatorial search at all, which is exactly what the staged picture of \S\ref{sec:rds} predicts --- cheap subset-signal-bearing stages first, leaving a small, heavily guided residual joint step. Consistently, per-part costs measured on real models by other protocols are of the same modest order and vary with circuit and regime ($\approx0.23$ from the clock; $0.47$--$0.62$ from regeneration; $\approx1.0$ from direct-$K$ control in prior work \citep{lei2026circuits}) --- a family- and circuit-indexed material constant, as the theory says it must be, and not a universal number the theorems could have supplied.
\emph{(B) Physical --- the nucleation reading, conditional on mixing.} For a diffusion of Step 1's form, the rate of a rare transition out of a metastable basin is, to leading exponential order, $J \asymp \exp(-\Delta S/\eps)$, with $\Delta S$ the quasipotential action of the least-improbable escape path --- Freidlin--Wentzell \citep{freidlin2012random,dembo1998large}; Kramers in the reversible case \citep{kramers1940brownian}. Given mixing on the formation timescale (Assumption~\ref{as:mixing}), large-deviation theory leaves no alternative leading-order form. \emph{Only this leg} requires that assumption, and only this leg supplies the thermodynamic vocabulary --- temperature, supersaturation, melting --- and hence the control operations of \S\ref{sec:control}. Rigorous escape-time and threshold analyses for online SGD exist in adjacent settings and scale exactly as this picture requires: sample complexity $N^{\max(1, k-1)}$ for information exponent $k$, with stretched-exponential trapping below threshold \citep{benarous2021online}. Three gaps between the classical theorems and production training are named rather than hidden: batch noise is not asymptotically small (the small-noise limit is a regime condition, whose diagnostic --- memorylessness --- is itself measured); adaptive optimizers carry moment state beyond a $\theta$-diffusion (our controlled measurements run under AdamW, so the measured signatures already include this deformation); and the landscape evolves during training (the quasi-stationarity clause of Assumption~\ref{as:mixing}).

\paragraph{Step 3 (proved here): for a $K$-part circuit, $\Delta S = \beta K$.}
This is the mathematical core of the paper --- the point where the generic exponential acquires its \emph{content} --- and it occupies \S\ref{sec:mathcore}.

\paragraph{Step 4 (measured): what multiplies the exponential.}
The prefactor factorizes as sites $\times$ attempt frequency $\times$ supersaturation under conditional independence (a mean-field approximation whose regime of validity, and measured breakdown at the boundary, are quantified in \S\ref{sec:synthesis}). The supersaturation has the measured form $\sigma(c) \propto (c - c_0)^{\gamma}$ above a floor concentration $c_0$ and is zero below it (\S\ref{sec:emergence}). The destruction term $D$ has two components with different supports: a \emph{reversion drift} on sub-critical embryos (always present --- it is what the drive must overcome, and what sets the floor), and a \emph{knockdown} component acting on formed structure, which is Arrhenius in an \emph{injected} temperature and dormant at native settings and practical learning rates (\S\ref{sec:control}).

\paragraph{The non-equilibrium branch (diagnosis cited; test and prescription ours).}
That SGD is out of equilibrium --- broken detailed balance, anisotropic state-dependent noise --- is established \citep{chaudhari2018stochastic,ziyin2026irreversibility}, and we do not claim the diagnosis. What we add is the \emph{operational} step: a fluctuation--dissipation \emph{qualification test} that any candidate ``temperature knob'' must pass, the measured result that all six native SGD knobs fail it (\S\ref{sec:notemp}), and the prescription that follows --- the drive is data, and a true temperature, if wanted for control, must be injected (\S\ref{sec:control}).

\section{Methods and apparatus}
\label{sec:methods}

All experiments fall into three apparatus classes; every number in the paper traces to one of them, and every protocol below was frozen (with pass/kill criteria) before the corresponding formal runs.

\textbf{Controlled toys (premise tests, EK suite; nose/melt/quench).} Models are 2-layer GPT-NeoX transformers (hidden 256, 4--16 heads, MLP $4\times$, rotary 25\%), trained with AdamW (lr $5\times10^{-4}$, $\beta=(0.9, 0.95)$, weight decay $0.01$, gradient clip $1.0$), batch 64, from seeded random initialization. \emph{Conjunction task} (EK-1/2/3): sequences of length 64 over a 64-token vocabulary follow a peaked Markov chain; with probability $\nu$ (the drive) a position is replaced by the sum mod 64 of the tokens at the $K$ fixed offsets --- the $K$-part target. The probe is next-token accuracy on fully-enabled sequences (256-sequence batch, evaluated every 10--25 steps); $\tstar$ = first crossing of $0.5$ (crossings at $0.3/0.5/0.7$ recorded). \emph{Open-loop query-parity task} (EK-4/5): sequences of length 96 over $\{0, 1\}$ plus query tokens; at query positions (gaps $\ge W{+}3$, guaranteeing every window contains only raw i.i.d.\ bits --- no feedback channel) the next token is the parity of the $K$ bits at $K$ of the $W$ window slots, the remaining slots being distractors. Parity makes every strict subset of the relevant bits exactly independent of the target. Score = accuracy at answer positions (chance $0.5$); $\tstar$ = first crossing of $0.75$ (crossings at $0.6/0.75/0.9$ recorded). \emph{Missing-parts grid} (EK-5): a neutral substrate (8{,}000 steps of constant-answer queries --- builds the readout, zero parts) is followed by a $+1$ curriculum, each stage trained to completion (score $\ge 0.9$ on three consecutive evaluations, plus 500 consolidation steps) and checkpointed; grid cells load the $s$-part substrate and train the $K$-part task with a 70/30 task/maintenance query mix (maintenance sustains the built parts; both the 30\% and 0\% variants are reported, with their respective identified artifacts, in \S\ref{sec:premise-tests}); slot roles are randomly permuted per seed so that slot difficulty averages out within every cell. \emph{Injected bath} (\S\ref{sec:control}): isotropic Gaussian weight noise of per-step scale $\sqrt{2 T \cdot \mathrm{lr}}$ added to all parameters --- the qualified temperature of Definition~\ref{def:fdt}.

\textbf{Statistics.} Waiting times are censoring-aware throughout: medians count censored runs as $+\infty$ (the Kaplan--Meier median, since all censoring occurs at the budget end; reported as undefined below 50\% formation), and means use the shifted-exponential maximum-likelihood estimator with right-censoring. The estimators were validated directly by rerunning every censored cell at doubled budget: the KM medians predicted the extended-budget measurements exactly in all three cases (the naive median-of-formed had been 15--19\% low). Calibration (smoke) seeds are excluded from all analyses.

\textbf{Public-checkpoint observation (\S\ref{sec:emergence}).} Mechanism-level probes are single forward passes at public Pythia and OLMo checkpoints in fp32 with eager attention (a stated numerical precondition: bf16 softmax underflow collapses the scores). Induction score = prefix-matching attention on repeated random tokens \citep{olsson2022context} (best head); precursor score = previous-token attention (best head); task probes are logit-difference accuracies with train/eval-disjoint content. $\tstar$ = first 0.5-crossing, log-interpolated between checkpoints; thresholds $0.3/0.5/0.7 \times$ linear/log interpolation give held-out median errors 4.4--13.0\% in all six variants. Clock constants were fit on five Pythia models, frozen, and only then evaluated on held-out models.

\textbf{Controlled continued training and from-scratch reproduction (\S\ref{sec:emergence}--\S\ref{sec:plasticity}).} Substrates are public Pythia models (70M--410M) or random initializations of the same architecture. The drive is supplied as repeated-token passages mixed into plain text at per-batch concentration $c$; the \emph{switch} replaces the mix at an experimenter-chosen step; \emph{aging} trains $S$ steps on plain data before the switch; \emph{resets} restore named weight groups (attention / MLP / embeddings) of an aged substrate to young values; the \emph{revival} sweep raises $c$ on a dead substrate. Arm sizes: switch/shelf/reset $n = 5$; floor sweep 0/12 and 4/5 at the critical concentrations; others $n = 2$--3 as flagged in \S\ref{sec:scope}. Every number is recomputable from released per-experiment JSON files and scripts.

\section{The mathematical core: a combinatorial barrier and its clock}
\label{sec:mathcore}

This section states the assumptions, proves the barrier and clock laws, and reports the experiments that test each assumption \emph{separately}. The logical status of every claim is marked, and the ledger now has three grades rather than one. \emph{Proved outright:} the no-partial-credit structure of the test apparatus (Lemma~\ref{lem:nopartial} --- Assumption~\ref{as:joint} is exact by construction there), and the exponential cost in $K$ for the task class, which stands on the query-complexity and leap theorems of \S\ref{sec:theory}(A) independently of any stochastic-analysis premise. \emph{Conditional theorems:} the nucleation-form propositions below, exact given Assumptions~\ref{as:mixing}--\ref{as:indep}; of these, only the \emph{physical reading} still leans on the unproven mixing clause, which no result of ours can currently discharge for transformer SGD. \emph{Measured:} every remaining premise, each falsifiable against a concrete alternative and tested in a dedicated preregistered experiment (\S\ref{sec:premise-tests}) independently of the conclusions it licenses.

\subsection{Setting and definitions}

\begin{definition}[Circuit, parts, order parameter]
\label{def:circuit}
A \emph{circuit} is a set of $K$ jointly necessary computational parts (attention-head functions) whose conjunction implements a capability. $K$ is defined operationally as the size of the \emph{minimal causal part set}: the smallest set of components each individually necessary at a just-post-ignition checkpoint (single-component ablation destroys the capability), machine-counted by the engine (\S\ref{sec:engine}), not assigned by hand. Its \emph{order parameter} $\phi(t)\in[0,1]$ is the normalized completeness (probe) score, and its formation time $\tstar$ is the first $0.5$-crossing of $\phi$ (threshold-robustness is quantified in \S\ref{sec:emergence}).
\end{definition}

\begin{assumption}[SDE regime and mixing --- the open frontier]
\label{as:mixing}
On the formation timescale, (a) the SDE approximation of Step 1 holds, and (b) the dynamics within the pre-formation basin is sufficiently mixing that the quasi-stationary escape theory of \citet{freidlin2012random} applies. Part (b) is \emph{not proved for transformer SGD}; it is the single unproved premise on which the propositions below are conditional, we flag it rather than hide it, and \S\ref{sec:scope} states what would change if it were established (every proposition below becomes an unconditional theorem).
\end{assumption}

\begin{assumption}[Simultaneous joint alignment]
\label{as:joint}
Formation of a $K$-part circuit requires the $K$ sub-alignments to be present \emph{simultaneously} in an attempt window; partial progress on a strict subset does not accumulate into the crossing. (Empirically discriminated from the alternative --- sequential, ``archivable'' assembly --- in \S\ref{sec:premise-tests}: the two models make opposite predictions, and the data select this one.)
\end{assumption}

\begin{lemma}[Zero drift on subset-supported directions --- exact in the test apparatus]
\label{lem:nopartial}
In the parity apparatus of \S\ref{sec:methods}, let the window bits $(b_i)$ be i.i.d.\ uniform and independent of the remaining context $Z$, let $R$ ($|R| = K$) be the relevant slots, and for any $S$ with $R \not\subseteq S$ let $\mathcal{A}_S = \sigma((b_i)_{i \in S}, Z)$ be the information available to a strict-subset feature. Then:
\textbf{(a)} $\mathbb{E}[Y \mid \mathcal{A}_S] = \tfrac12$ a.s.;
\textbf{(b)} among all $\mathcal{A}_S$-measurable predictors, the population cross-entropy is \emph{globally minimized} at chance ($p \equiv \tfrac12$) --- no descent below chance exists anywhere in the subset-supported predictor class;
\textbf{(c)} (Fourier form) the target's expansion over the window bits consists of the \emph{single} character $\chi_R$: $\mathbb{E}[(2Y{-}1)\chi_T] = 0$ for every $T \neq R$, so the population gradient's correlation with any partial character vanishes identically;
\textbf{(d)} (stationarity) on the chance manifold, the projection of the population loss gradient onto the $\mathcal{A}_S$-measurable function subspace is zero, for every strict subset $S$ simultaneously.
Assumption~\ref{as:joint} is therefore not an idealization in the apparatus but \emph{exact by construction}, at the level of the dynamics and not merely of the information.
\end{lemma}

\begin{proof}
(a) $Y = \bigoplus_{i \in S \cap R} b_i \oplus \bigoplus_{j \in R \setminus S} b_j$; the second factor is uniform and independent of $\mathcal{A}_S$ since $R \setminus S \neq \emptyset$. (b) Conditioning on $\mathcal{A}_S$, $\mathbb{E}[\ell(p, Y) \mid \mathcal{A}_S] = \tfrac12\ell(p, 1) + \tfrac12\ell(p, 0)$, uniquely minimized at $p = \tfrac12$; take expectations. (c) $2Y - 1 = \chi_R$, and characters are orthonormal under the uniform measure. (d) For any $\mathcal{A}_S$-measurable function-space direction $g$, the cross-entropy functional gradient component is $\mathbb{E}[(p - Y) g] = \mathbb{E}[(p - \mathbb{E}[Y \mid \mathcal{A}_S]) g] = \mathbb{E}[(p - \tfrac12) g]$, which vanishes at $p \equiv \tfrac12$.
\end{proof}

\begin{remark}[Dynamical reading]
Population drift can neither build nor be rewarded for any strict-subset feature (b, d); all progress toward the target flows through the joint-$K$ channel (c); minibatch fluctuation is the only mover of subset features and is unrewarded. This is precisely the ``no feedback from guessed indices unless the subset is exactly correct'' property that \citet{barak2022hidden} identify for sparse parity, here established for our apparatus by construction --- the leap-$K$ property \citep{abbe2023leap} instantiated.
\end{remark}

\begin{remark}
For real circuits, Assumption~\ref{as:joint} is an empirical property rather than a construction, and it is measured, not presumed: pre-ignition, the target capability's own curve is flat to slopes of $10^{-6}$--$10^{-5}$/step and a placebo sweep of 200 random heads carries no predictive signal (42.5\% error = guessing; \S\ref{sec:emergence}) --- the behavioral face of ``no partial credit.'' Real assembly is \emph{staged} precisely where partial credit does exist (parts with independent utility form early, \S\ref{sec:rds}); the assumption governs the rate-limiting final step, which is where it is tested.
\end{remark}

\begin{assumption}[Part independence, with a measured failure domain]
\label{as:indep}
Away from the capacity wall, the $K$ sub-alignment events in an attempt window are approximately independent. This is the mean-field idealization of the theory; its failure \emph{is itself measurable} (Corollary~\ref{cor:dependence}) and is measured in \S\ref{sec:premise-tests}.
\end{assumption}

\subsection{The barrier law}

\begin{proposition}[Additivity of the combinatorial barrier]
\label{prop:additivity}
Under Assumptions~\ref{as:mixing}--\ref{as:indep}, let the $i$-th sub-alignment be a rare event of per-attempt probability $p_i \asymp e^{-s_i/\eps}$. Then the per-attempt formation probability satisfies
\[
P(\text{formation in one attempt}) \;=\; \prod_{i=1}^{K} p_i \;\asymp\; \exp\!\Big(-\tfrac{1}{\eps}\sum_{i=1}^{K} s_i\Big),
\]
so the escape action is additive, $\Delta S = \sum_i s_i$; for exchangeable parts ($s_i \equiv s$),
\[
J \;\propto\; \nu\, e^{-\beta K}, \qquad \beta := s/\eps .
\]
\end{proposition}

\begin{proof}
By Assumption~\ref{as:joint} the formation event in an attempt window is the intersection $\bigcap_i E_i$ of the $K$ sub-alignment events; by Assumption~\ref{as:indep} the probability of the intersection is the product of the probabilities. Taking logarithms, the large-deviation exponents add --- the elementary additivity of independent rare-event costs that underlies Cram\'er's and Sanov's theorems \citep{dembo1998large}. Exchangeability collapses the sum to $sK$. Under Assumption~\ref{as:mixing}, the escape rate inherits this action through the Freidlin--Wentzell principle (Step 2), giving $J \propto \nu e^{-\beta K}$ with the attempt frequency $\nu$ as prefactor.
\end{proof}

\begin{remark}
The mathematical content is deliberately elementary; the \emph{scientific} content is the identification (Assumption~\ref{as:joint}) and the independence idealization (Assumption~\ref{as:indep}), both of which are tested as separate experimental questions in \S\ref{sec:premise-tests} rather than bundled into a fit. This is why the exponential-in-$K$ form is \emph{derived, not chosen}: given the premises, no other form is available, and the premises are checked independently of the conclusion.
\end{remark}

\begin{corollary}[Deviation from linearity measures part dependence]
\label{cor:dependence}
If the parts share a common enabling cause (positive association), $P(\bigcap_i E_i) \ge \prod_i P(E_i)$, so $\Delta S \le sK$: the barrier grows \emph{sublinearly} in $K$. If forming parts consume a shared resource (negative interaction, e.g.\ capacity near the wall), the barrier grows \emph{superlinearly}. The measured departure of $\ln \tstar$ from linearity in $K$ is therefore a direct, signed measure of part dependence --- the theory's own error bar, not an anomaly.
\end{corollary}

\subsection{The clock law}

\begin{proposition}[Memoryless waiting and the precursor clock]
\label{prop:clock}
Under Assumption~\ref{as:mixing}, the normalized escape (formation) time is asymptotically exponentially distributed with mean $\propto 1/J$ \citep{day1983exponential,bovier2015metastability}. For a composition of $K$ exchangeable parts and its single-part precursor driven by the same stream, the exact ratio is
\[
\frac{\mathbb{E}\,\tstar_{\rm comp}}{\mathbb{E}\,\tstar_{\rm prec}}
\;=\; \frac{A_{\rm prec}}{A_{\rm comp}}\, e^{\beta (K-1)},
\]
where $A = N\nu$ collects each circuit's prefactor (sites $\times$ attempt frequency). Under the additional prefactor-cancellation assumption $A_{\rm prec} \approx A_{\rm comp}$ (A3, stated and tested below), the clock $C(K) = e^{\beta(K-1)}$ is a \emph{derived} constant: measuring the precursor's arrival dates the composition's arrival before the composition exists.
\end{proposition}

\begin{proof}
The exponential limit law for exit times from a metastable basin under small noise is classical \citep{day1983exponential}; mean exit time $\asymp$ inverse rate is the content of metastability theory in the same regime \citep{bovier2015metastability}. The ratio of means is then the ratio of inverse rates, which is $({A_{\rm prec}}/{A_{\rm comp}})\,e^{\beta(K-1)}$ by Proposition~\ref{prop:additivity}; A3 reduces it to $e^{\beta(K-1)}$.
\end{proof}

\begin{remark}[The prefactor ratio is counted, not assumed]
\label{rem:prefactor}
The prefactor is where a skeptic should push --- ``a single head and a multi-head composition have search spaces differing by orders of magnitude'' --- so we \emph{count} it rather than assume it. The prefactor of a rate-limiting step is $A = N_{\rm sites}(\text{missing structure}) \times \nu$, where $N_{\rm sites}$ counts eligible configurations \emph{of the missing structure only}: the missing-parts grid (EK-5, \S\ref{sec:premise-tests}) measures directly that built structure is kinetically invisible, so what enters is not the circuit's total configuration space but the remainder's. For the clock, the precursor's missing structure is one part among the $H_{\ell_1}$ heads of its layer, and the composition's post-precursor missing structure is one part among the $H_{\ell_2}$ heads of the reading layer:
\[
\frac{A_{\rm prec}}{A_{\rm comp}} \;=\; \frac{H_{\ell_1}}{H_{\ell_2}} \;=\; 1 \quad \text{for uniform architectures (equal heads per layer, as in Pythia and OLMo).}
\]
The pair-counting intuition ($A_{\rm comp} \propto H_1 H_2$, or more generally a $\binom{H}{K}$ combinatorial factor) prices forming \emph{both} parts jointly from nothing --- that is the $s{=}0$ column of the grid, and it is carried by the barrier $e^{-\beta K}$, not by the clock's prefactor. This is not a verbal distinction but a measured one, and the missing-parts grid is precisely the experiment that decides it: were the prefactor to carry a combinatorial factor over the \emph{whole} part set, cells with more already-built parts would search a systematically different space and rows of fixed missing count could not be flat --- yet they are, to $1.19$--$1.37$ at $K{-}s{=}3$ across $K = 3, 4, 5$ (\S\ref{sec:premise-tests}), with overlapping bootstrap intervals. Two further measurements support the counting model with no new assumptions: sites enter the rate linearly (EK-3, deep-barrier slope $-1.25$); and two different circuits ($K{=}2$, $C{=}1.24$; $K{=}4$, $C{=}2.1$) are consistent with one $\beta \approx 0.23$ --- two equations, one constant, impossible under freely varying prefactors. The model is falsifiable where it should be: architectures with \emph{unequal} per-layer head counts are predicted to shift the clock by $\ln(H_{\ell_1}/H_{\ell_2})$, a concrete out-of-family test. Residual imbalance is absorbed into the effective $\beta$, part of why $\beta$ is family-indexed.
\end{remark}

\begin{remark}[What is fit and what is not, and the range-of-$K$ question]
$\beta$ is a family-dependent \emph{effective} constant (the general $\sum_i s_i$ form explains its variation across capabilities and families); it is fit \emph{once, on-axis} and then frozen. The clock's \emph{form} is derived; its prospective, frozen-constant validation on held-out models (median 5\% error, 6/6, cross-family) is reported in \S\ref{sec:emergence}. We are explicit about the dynamic range: our own direct-$K$ measurements span $K = 2..4$, over which exponential, power, and linear fits are locally indistinguishable ($R^2$ 0.980/0.978/0.989) --- too narrow for model selection by fit alone, which is exactly why the premises, not the curve, carry the argument (\S\ref{sec:premise-tests}). Wide-range evidence exists in an adjacent literature: for $k$-sparse parity --- the canonical task in which \emph{no strict subset carries signal}, i.e.\ a pure joint-alignment problem --- SGD's time-to-learn scales as $d^{\Theta(k)}$ \citep{barak2022hidden}: $\ln \tstar$ linear in $k$ with per-part cost $\ln d$, the additive combinatorial barrier over a $k$-range far wider than ours. This is Proposition~\ref{prop:additivity} instantiated in an independent setting, not our data; extending our own ladder to $K \ge 5$ with a subset-signal-free design is the single highest-value additional experiment this theory admits.
\end{remark}

\subsection{Testing the premises, not just the conclusion}
\label{sec:premise-tests}

Each assumption above is the target of a dedicated experiment (design frozen before running; the EK suite):

Statistics are censoring-aware throughout: medians count censored runs as $+\infty$ (the Kaplan--Meier median here, since all censoring occurs at the budget end), and means use the shifted-exponential maximum-likelihood estimator with right-censoring; cells with $<$50\% formation are reported as undefined rather than estimated. The estimators were validated directly: rerunning every censored cell at doubled budget (to 92--100\% formation), the censoring-corrected KM medians predicted the extended-budget measurements \emph{exactly} in all three cases, and the MLE means to 0.2--10\%, where the naive median-of-formed had been 15--19\% low. $K{=}1$ cells form at the evaluation-grid floor ($\tstar \le 25$) and are therefore reported only as bounds, never entered into fits.

\textbf{(EK-1) Memorylessness of waiting times} [tests the premise of Prop.~\ref{prop:clock}]. At fixed $(K,\nu)$, the formation-time distribution over 60 seeds per cell spans the predicted crossover (Fig.~\ref{fig:ek}a): at shallow barriers the waiting is nearly deterministic (raw $\kappa \approx 220$, CV 0.07), and as the barrier deepens the shifted-tail shape parameter falls to $\kappa = 2.0$ and then \emph{plateaus at the memoryless limit} --- $\kappa = 1.2$ and $1.3$ at the two deepest barriers (60/60 formed at extended budget) --- the exponential regime, exactly where the clock law needs it, and stable once reached. (Cell labels: the deep-barrier cells here are the summed design's nominal $K{=}3$, which the discriminator below shows solves the task by a two-part shortcut. The measured claim is unaffected --- what is established is that waiting becomes memoryless once formation is barrier-dominated, a statement about the kinetics and not about the part count --- but we do not attach a part count to these cells.) Consistent prior evidence: $\kappa \approx 80$ at full batch $\to$ $\kappa \to 1$ under minibatch noise \citep{lei2026circuits}. \emph{Logical status:} exponential waiting is a \emph{necessary condition} of Assumption~\ref{as:mixing}(b), so EK-1 is a consistency test that could have falsified the assumption and did not --- it is not, and is not offered as, a proof of mixing. \emph{Regime diagnosis:} $\kappa$ also arbitrates between the two proven mechanisms available to the computational leg. \citet{barak2022hidden} show that at population scale SGD amplifies a deterministic Fourier-gap signal of size $\gamma = \Theta(n^{-(k-1)/2})$ --- progress hidden but not random --- while at batch sizes far below the $\sim n^k$ needed to resolve $\gamma$ per step, the search is fluctuation-dominated. The two mechanisms share the same $n^{\Theta(k)}$ cost (the SQ bound is mechanism-blind) but differ in waiting-time shape: deterministic amplification gives $\kappa \gg 1$, fluctuation-driven escape gives $\kappa \to 1$. Our measured $\kappa$ spectrum ($\approx$220 at shallow barriers $\to$ 1.2--1.3 at deep ones, at batch 64) locates the apparatus in the fluctuation-dominated regime precisely where the nucleation reading applies --- and simultaneously shows the crossover toward the deterministic regime at shallow barriers, reconciling the two literatures rather than contradicting either.

\textbf{(EK-2) Additivity, and a measured pathway change} [tests Prop.~\ref{prop:additivity} and Cor.~\ref{cor:dependence}]. The \emph{same nested part set} is assembled into $K = 2, 3$ at fixed drive across three widths (Fig.~\ref{fig:ek}b). We report this family in full, including a contamination we found by direct probing and which withdraws part of it.

In the summed design the target at an enabled position is the sum of the $K$ preceding tokens; when two consecutive positions are both enabled, subtracting the two relations gives $s_t = 2 s_{t-1} - s_{t-K-1}$ --- an \emph{algebraic shortcut requiring only two attention targets, independent of $K$}. It is a second solution to the same task: it agrees with the true $K$-sum wherever the recurrence holds and disagrees where it is broken. A discriminating probe (scoring positions whose predecessor is \emph{not} enabled, so that the shortcut is wrong and the true sum still right) settles which solution each cell learns. The verdict is sharp: at $K{=}2$ the network learns the true sum (broken/valid accuracy ratio $0.94$, $n{=}8$), while at $K{=}3$ and $K{=}4$ it learns the shortcut (ratios $0.11$ and $0.03$; broken-position accuracy $0.060$ and $0.015$ against a chance level of $0.016$). \emph{We therefore withdraw the summed family's $K \ge 3$ cells as barrier measurements}: $\hat\beta = \ln(\tstar_3/\tstar_2) = 1.75$ compares a shortcut-formation time to a true-sum formation time and is not a clean per-part increment. The trustworthy per-part cost is the parity family's ($\ge$11$\times$; \S\ref{sec:premise-tests}, EK-4), and the reason for the gap is now mechanistic rather than mysterious: \emph{the summed design admits a shortcut and the parity design provably does not}.

The pathway change is itself a measured result, and it disposes of what previously looked like an anomaly. $K{=}4$ forming \emph{faster} than $K{=}3$ (1175 vs.\ 1738 in replication; 1138 vs.\ 1725 originally) is not a violation of additivity: from $K{=}3$ upward both cells run on the same $K$-independent shortcut, whose cost does not grow with $K$, so the ladder flattens rather than climbing. This is textbook parallel-pathway competition --- kinetic versus thermodynamic product --- with the crossover located at $K{=}3$: below it the true solution is the cheaper route, above it the two-part shortcut is, and the system takes whichever barrier is lower, not whichever solution is more accurate. The picture makes a sharp prediction, which we then tested: if every cell from $K{=}3$ upward runs on the same $K$-independent route, formation times must \emph{plateau} rather than continue climbing. Measured (8 seeds per cell): $300$, $1738$, $1175$, $938$, $962$ steps at $K = 2, 3, 4, 5, 6$, with the discriminator ratio $0.94$ at $K{=}2$ and $0.11, 0.03, 0.03, 0.03$ thereafter. The ladder flattens exactly where the discriminator says the route changes, and the size of the effect is not marginal: an additive ladder continued from the true-sum cells would place $K{=}6$ near $3\times10^{5}$ steps, against $962$ measured --- a factor of $350$. A confirmed prediction of a pathway crossover is stronger evidence for parallel-route kinetics than the original clean ladder would have been for additivity. Our first hypothesis for this cell, a metastable partial-sum intermediate (Ostwald staging), was tested by probing every partial sum during training and \emph{refuted}: 0/8 seeds showed a partial-sum plateau (peak partial accuracy $0.03$ against chance $0.016$). We report that failure because it is what forced the correct explanation.

What the width sweep still shows, within the uncontaminated $K{=}2$ cells: shrinking width toward the capacity wall inflates the barrier ($\tstar_2$: $300 \to 900 \to 2000$ from $d{=}256$ to $96$ to $64$, a $6.7\times$ capacity penalty) --- the capacity-competition side of Cor.~\ref{cor:dependence}, measured on cells the discriminator certifies as true-sum solutions.

\textbf{(EK-3) Product law versus bottleneck} [tests Assumption~\ref{as:joint}]. Simultaneous joint alignment predicts rate $\propto$ sites in the rare-event regime; sequential/archivable assembly predicts a bottleneck (max) law whose site-dependence should \emph{strengthen} as the drive stops binding. A sites(heads) $\times$ drive($\nu$) grid decides (Fig.~\ref{fig:ek}c): at the memoryless point the log--log slope of $\tstar$ against head count is $-1.25$ (confirmed uncensored at extended budget, 8/8) --- rate proportional to sites, the product law; in the drift-limited cells the slope collapses to $\approx 0$ --- no site statistics where formation is not a rare event. (As in EK-1, the deep-barrier cells are the summed design's nominal $K{=}3$; the site-proportionality is a statement about the kinetics of a barrier-dominated formation event, with no part count attached.) The bottleneck law predicts the \emph{opposite} pattern (site-dependence strongest at high drive) and is refuted in both regimes. Consistent prior evidence: $J \propto N \cdot r$, $R^2$ 0.97 \citep{lei2026circuits}.

\textbf{(EK-5) The missing-parts grid} [tests additivity through its sharpest fingerprint]. If barriers add, formation time depends only on the number of \emph{missing} parts, never on the circuit's total size: a substrate holding $s$ of $K$ parts (pre-built by staged curriculum, each stage trained to completion --- partially formed parts do not count, a threshold measured in prior work \citep{lei2026circuits}) waits only for the remaining $K-s$ to align jointly. A $(K,s)$ grid over $K = 2..5$, $s = 0..4$ (46 cells, 8--16 seeds each, slot roles randomized per seed, criteria preregistered) confirms the fingerprint where it is cleanest: at fixed $K{-}s = 3$ the cell medians for $K = 3, 4, 5$ agree to 1.19--1.37 (4550/6225/5412 steps; bootstrap 95\% intervals share a common overlap, $[1275, 8050]$), and at $K{-}s = 4$ to 1.08--1.34 (common overlap $[4912, 14412]$) --- \emph{a five-part circuit missing three waits as long as a three-part circuit missing three}, though its total barrier is far larger. The alternatives fail characteristically: an energetic (cluster-size) barrier requires within-row growth with $K$ (none observed; row maxima are interior), and sequential accumulation requires arithmetic row spacing (measured spacings are not arithmetic). Two protocol artifacts were caught, mechanism-traced, and controlled: a copy-task maintenance stream interferes specifically with $s{=}1$ cells (removing it restores 8/8 formation and a $>$2$\times$ lower median), while zero maintenance lets built parts decay in high-$s$ cells, in proportion to their exposure time --- under per-cell clean protocols the shallow-row spread is 1.86, and the deepest measured row is flat under \emph{both} protocols (1.13 and 1.33). A window-width sweep closed the last gap: apparent rung-height saturation at deep rows vanishes when distractor slots are replenished (rung ratio $\times2.1 \to \times3.9$ from $W{=}5$ to $6$, exceeding the entire budget at $W \ge 7$) --- a designed-in ceiling of the smallest window, not a failure of additivity (Fig.~\ref{fig:ek5}).

\begin{figure}[htbp]\centering
\includegraphics[width=\linewidth]{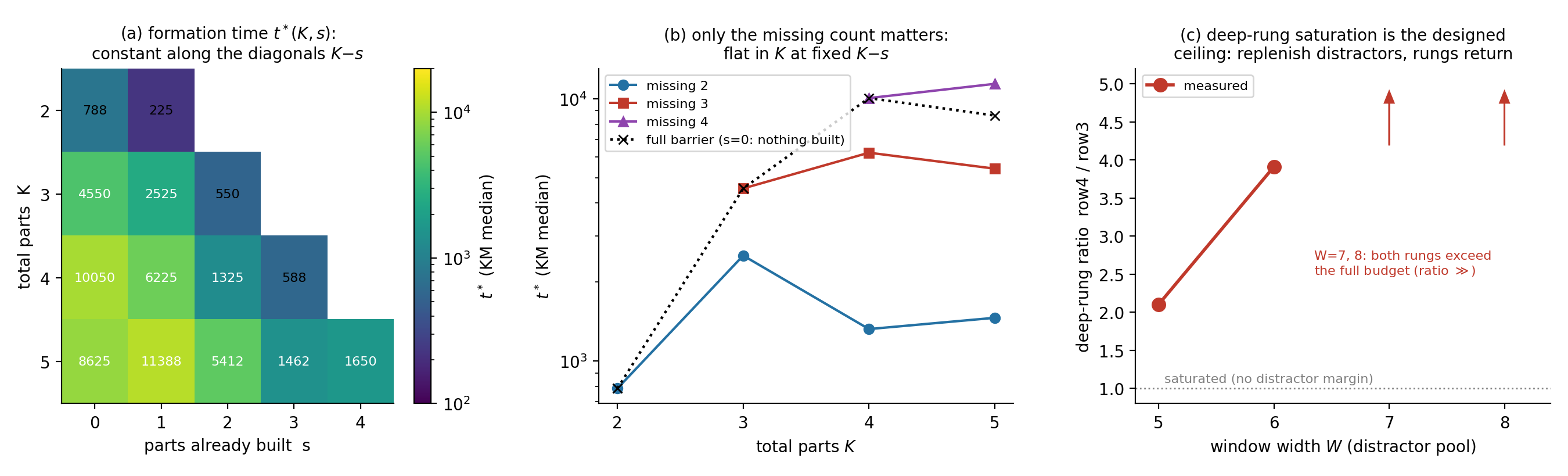}
\caption{\textbf{The missing-parts grid} (EK-5; preregistered; censoring-aware medians, uniform-protocol matrix). (a) Formation time $t^*(K,s)$ is constant along the diagonals of fixed $K{-}s$: what has been built is kinetically invisible. (b) The same data against total $K$: rows at fixed missing count are flat while the full barrier ($s{=}0$ column, dotted) climbs $\sim$13$\times$ --- the direct rate-determining-step measurement (\S\ref{sec:rds}). (c) The deep-rung saturation of the $W{=}5$ window is a designed-in ceiling: replenishing distractor slots restores the rung ($\times2.1 \to \times3.9$ at $W{=}6$; both rungs exceed the full budget at $W \ge 7$).}
\label{fig:ek5}
\end{figure}

\textbf{The regime map as a consistency check.} Three independent measurements draw the same boundary: waiting is memoryless ($\kappa \to 1$) precisely in the cells where the site-product law holds (slope $-1.25$), and deterministic ($\kappa \gg 1$) precisely where site statistics vanish (slope $\approx 0$). The nucleation theory's signatures appear together, in its claimed domain --- the barrier-dominated regime that real compositional circuits inhabit --- and switch off together outside it. \textbf{(A3, shared pace)} Across from-scratch ladder seeds, successive stage increments are strongly co-scaled ($\mathrm{corr}=+0.89$) --- precisely the common mode the ratio clock cancels.

\begin{figure}[htbp]\centering
\includegraphics[width=\linewidth]{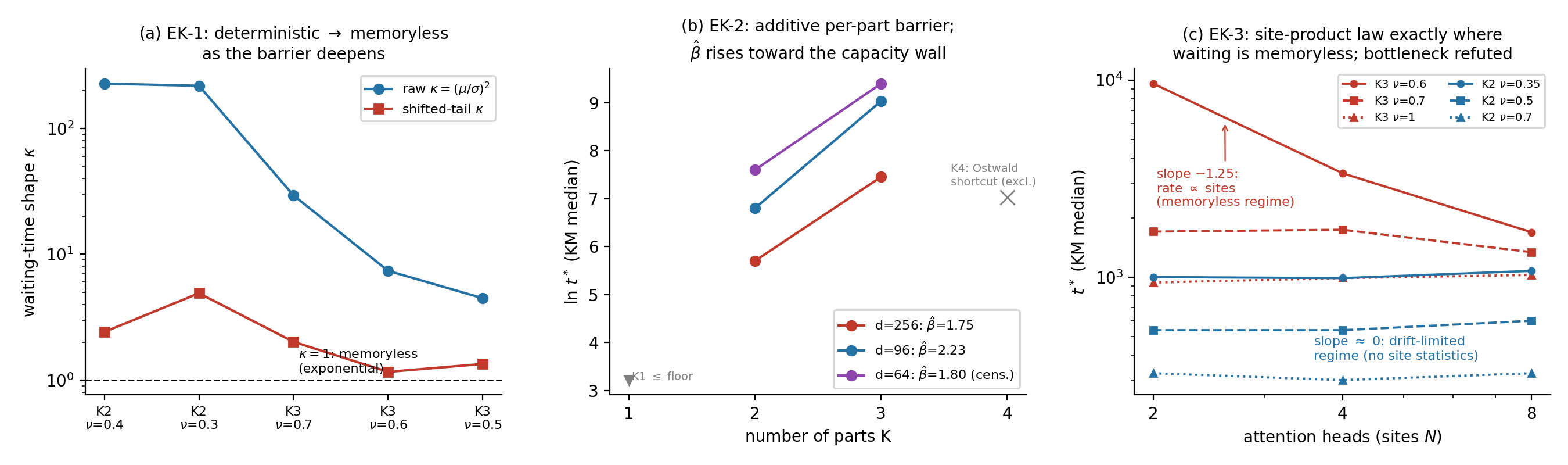}
\caption{\textbf{Testing the premises, not the conclusion} (EK suite; 46 preregistered cells including extended-budget reruns, censoring-aware statistics). (a) Waiting-time shape crosses from deterministic ($\kappa \approx 220$) to memoryless and plateaus there ($\kappa = 1.2, 1.3$ at the two deepest barriers) --- the premise of Prop.~\ref{prop:clock}. (b) The per-part barrier increment is additive at roomy width ($\hat\beta = 1.75$, both estimators) and rises toward the capacity wall --- Prop.~\ref{prop:additivity} and the superlinear dependence of Cor.~\ref{cor:dependence}. (c) Rate $\propto$ sites exactly where waiting is memoryless; no site statistics in the drift-limited regime; the bottleneck alternative predicts the opposite pattern and is refuted --- Assumption~\ref{as:joint}.}
\label{fig:ek}
\end{figure}

Near-wall breakdown of factorization (the 112\% point, error growing as a power of floor distance) is reported with the synthesis in \S\ref{sec:synthesis} --- by Cor.~\ref{cor:dependence} it is the predicted signature of part dependence at shared capacity, and it converts the law's boundary from a caveat into a falsifiable commitment.

\section{No intrinsic temperature: the fluctuation--dissipation qualification test}
\label{sec:notemp}

Equilibrium escape theory presumes a bath. Before Eq.~\eqref{eq:master} can be read as \emph{thermal} nucleation, one must ask which, if any, of training's native knobs supplies a temperature. We make the question operational.

\begin{definition}[Temperature qualification]
\label{def:fdt}
A noise source qualifies as a \emph{temperature} for the escape dynamics if it satisfies the fluctuation--dissipation structure of an equilibrium bath \citep{kubo1966fluctuation}: its covariance is (i) isotropic, (ii) independent of the current state (in particular, it does not vanish at loss minima), so that (iii) the stationary distribution it induces is the Gibbs measure of the loss at that temperature, and Kramers/Freidlin--Wentzell escape applies with $\eps = T$.
\end{definition}

This is a \emph{criterion}, cited, not a theorem of ours; our contribution is to apply it as a qualification exam. Two clarifications delimit the claim. First, it concerns the \emph{equilibrium} notion of temperature --- the one under which Gibbs sampling and Kramers escape are licensed; non-equilibrium effective temperatures (noise temperatures, FDT-violation ratios) can of course be \emph{defined} for SGD, and our statement is precisely that none of them earns the escape-theory role. Second, the verdict is empirical and unambiguous. Minibatch-gradient noise has covariance proportional to the per-sample gradient second moment: state-coupled (it dies where gradients die), highly anisotropic (covariance rank as low as $\sim$1\% \citep{chaudhari2018stochastic}), and therefore fails (i)--(iii) structurally --- the known result that SGD does not sample a Gibbs measure \citep{chaudhari2018stochastic}, sharpened recently to full thermodynamic irreversibility with the explicit warning that step size is not a temperature \citep{ziyin2026irreversibility}. We tested all six native knobs (minibatch noise, learning rate, batch size, dropout, weight decay, and their clock-rescaled combinations) against Definition~\ref{def:fdt} in the controlled apparatus of \S\ref{sec:control}: \emph{zero qualify}. Learning rate is not even ``partially'' thermal: converting its noise amplification into an effective clock leaves a monotonic $60\times$ degradation and zero net barrier assistance --- it is an \emph{anti-thermal amplifier}, and its noise power is strongly state-dependent (measured factor $2.5\times$ between formed and plateau states for SGD noise; $72\times$ for dropout). This is consistent, not in tension, with the real-LM finding that within the igniting region learning rate simply accelerates ignition (\S\ref{sec:emergence}): in both regimes the unified statement is that \emph{lr is a clock, never a temperature} --- it rescales time without ever supplying barrier assistance. Anti-thermal noise does not merely fail to help: because it is gradient-coupled, it selectively dismantles exactly the multi-part coincidences that Proposition~\ref{prop:additivity} prices --- measured as barrier-gated damage (high-$K$ circuits hurt $\approx 2.7\times$ more than low-$K$ on the dropout axis).

Two consequences organize the rest of the paper. First, \emph{the drive is data}: with no thermal term available, the only state-decoupled generation knob in Eq.~\eqref{eq:master} is the supersaturation $\sigma(c)$ --- the concentration of capability-relevant statistics in the stream. Formation in native training is \emph{athermal, driven nucleation}, and its control variable is the data mix (\S\ref{sec:emergence}). Second, \emph{the missing term can be installed}: an injected isotropic weight-noise bath passes the qualification by construction, and \S\ref{sec:control} shows that installing it buys exactly the control functions (anneal, dissolve, pin) that native training lacks. In prior work the injected bath was verified to restore a single $(U,D)$ pair reproducing drift, occupancy, and committor simultaneously \citep{lei2026circuits}.

\section{Reading the equation forward: emergence}
\label{sec:emergence}

Forward means: hold a substrate fixed, raise the drive $\sigma(c)$, and watch the generation term. The equation predicts a threshold (no nucleation below a supersaturation floor), a throttle (rate rising steeply above it), and a clock (the barrier ratio of Prop.~\ref{prop:clock}). Each is confirmed, the clock prospectively.

Probes, protocols, and preregistration are specified in \S\ref{sec:methods}; mechanism-level probes make ``emergence'' unambiguous, and constants were frozen before any held-out model was measured.

\subsection{The clock, prospectively validated}

Fit $C = 1.241$ on five Pythia models (leave-one-out median 3\%); froze constants; predicted six unseen models (five deduped-data variants and pythia-1b, outside the fitted scale range): 6/6 within the preregistered 25\% band, median 5.0\% (Fig.~\ref{fig:clock}). Cross-family: OLMo-1B --- different architecture, data, tokenizer --- measures 1.244 against the frozen 1.241. A cheap variant reads a single attention score at step 512 (0.36\% of training, target capability at baseline everywhere): 6/6, median 12.8\%. \emph{No circularity, twice over:} every prospective test used the human-known previous-token head from prior literature; the automatic engine (\S\ref{sec:engine}) is validated against these answers, not their source. And the prospective record concerns \emph{induction only}, with constants frozen before any held-out model was measured; the cross-circuit unification of $\beta$ (induction with IOI, \S\ref{sec:engine}) is a separate, explicitly post-hoc consistency check and is nowhere counted as a prospective result. \emph{No alternative signal:} pre-ignition, the target's own curve is flat (slopes $10^{-6}$--$10^{-5}$/step), and a placebo sweep shows 200 random heads predict at 42.5\% error (= guessing) versus the causal precursor's 8.7\%. Where \citet{lavie2026phase} conclude softmax-attention emergence is intrinsically unpredictable and \citet{baherwani2026random} that timing is initialization-random, these prospective results are a direct counterexample for the capability class covered here: the timing is predictable, from the precursor, because the barrier is combinatorial.

\begin{figure}[htbp]\centering
\includegraphics[width=\linewidth]{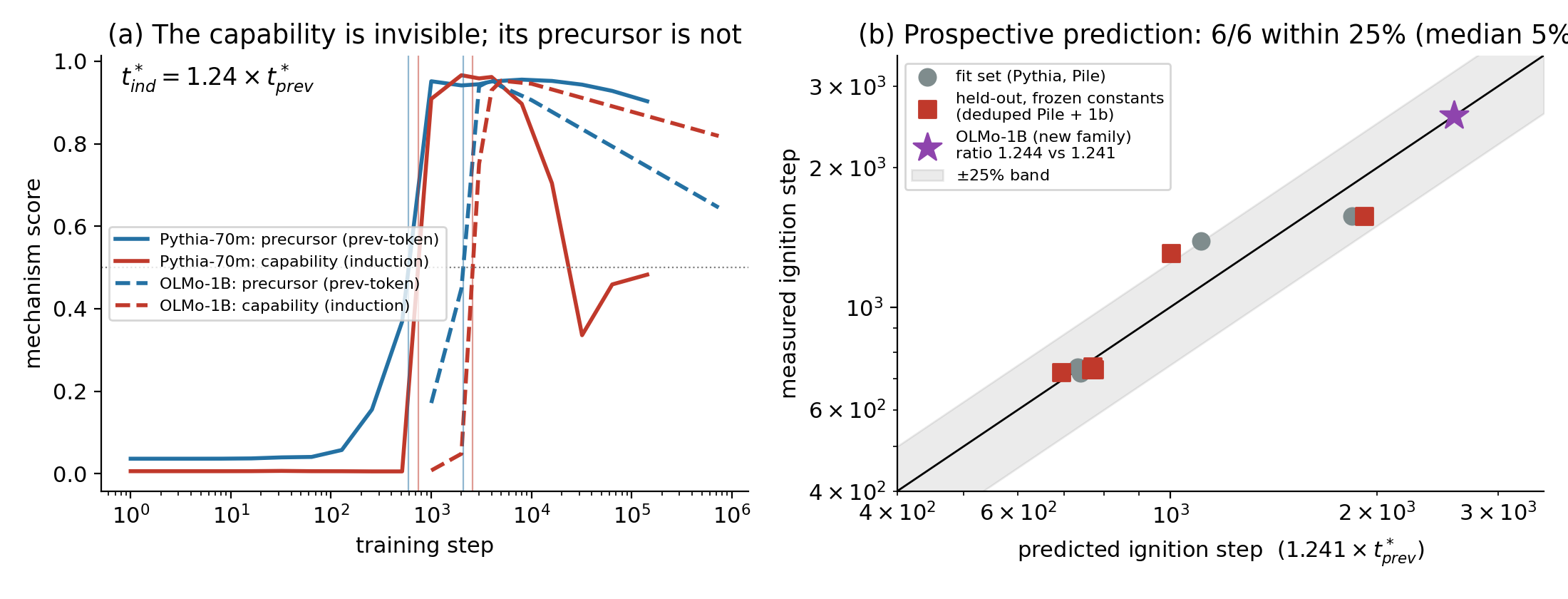}
\caption{\textbf{The clock.} (a) The capability is invisible while its precursor is not; the same two-stage shape reproduces on OLMo (dashed). (b) Prospective prediction with frozen constants: 6/6 held-out models within $\pm$25\% (median 5\%); OLMo-1B at 1.244 vs.\ frozen 1.241.}
\label{fig:clock}
\end{figure}

\textbf{Mechanism dates behavior.} Across 15/15 model--task pairs the mechanism ignites first; real-text in-context learning ignites simultaneously with the mechanism (ratios 0.93--0.97), strict exact-copy behavior lags $\times$1.27 (Fig.~\ref{fig:bridge}). A task-level few-shot skill (in-context label induction) shows the full ladder on public checkpoints: never below the capacity wall (70M flat through 143k steps), late and stretched at its edge (160M $\sim$12{,}400), scale-consistent above it (410M: 3{,}933; 1B: 3{,}822), $\sim$5$\times$ after the induction mechanism (Fig.~\ref{fig:task}). And on all choice-based probes the pre-ignition period is systematically \emph{below chance} --- an anti-phase reported as a robust correlational early-warning channel, \emph{not} a proven mechanism: the direct causal test was inconclusive because the precursor is redundant (\S\ref{sec:negative}).

\begin{figure}[htbp]\centering
\includegraphics[width=\linewidth]{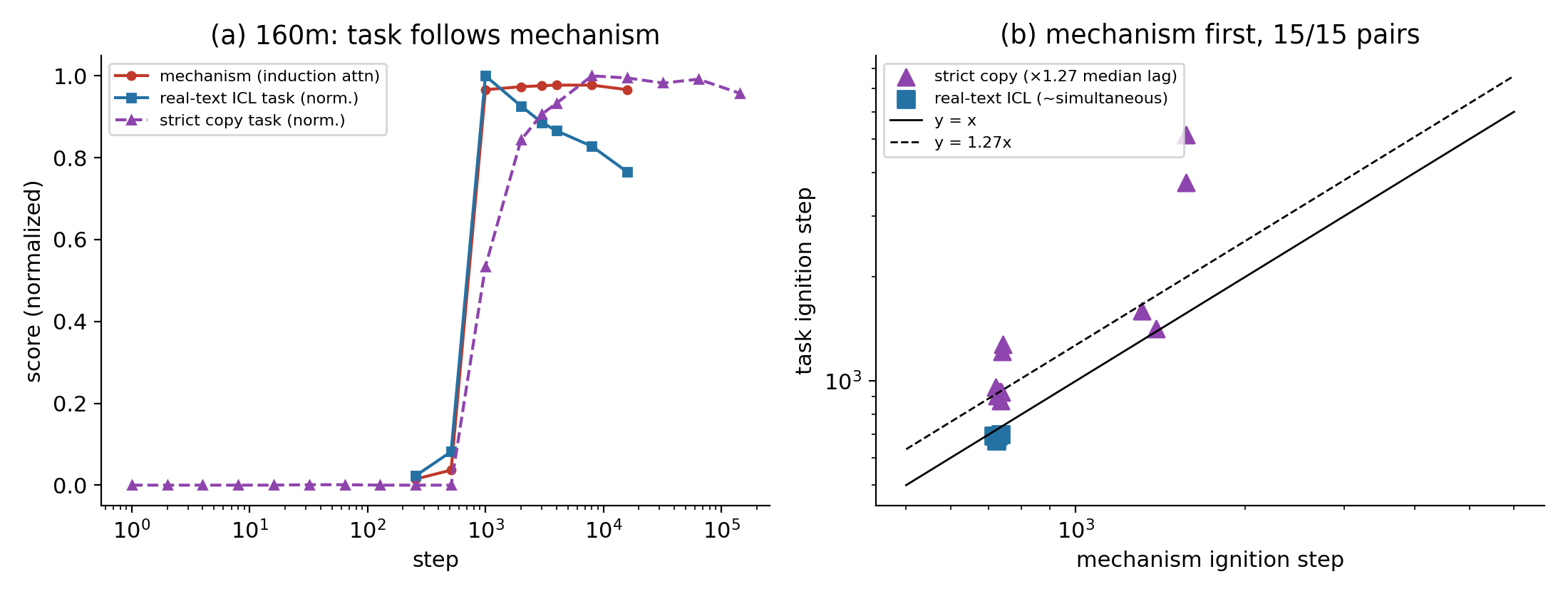}
\caption{\textbf{The bridge.} (a) At 160M, real-text in-context learning ignites with the mechanism; strict copy lags. (b) Mechanism-first in 15/15 pairs.}
\label{fig:bridge}
\end{figure}

\begin{figure}[htbp]\centering
\includegraphics[width=0.8\linewidth]{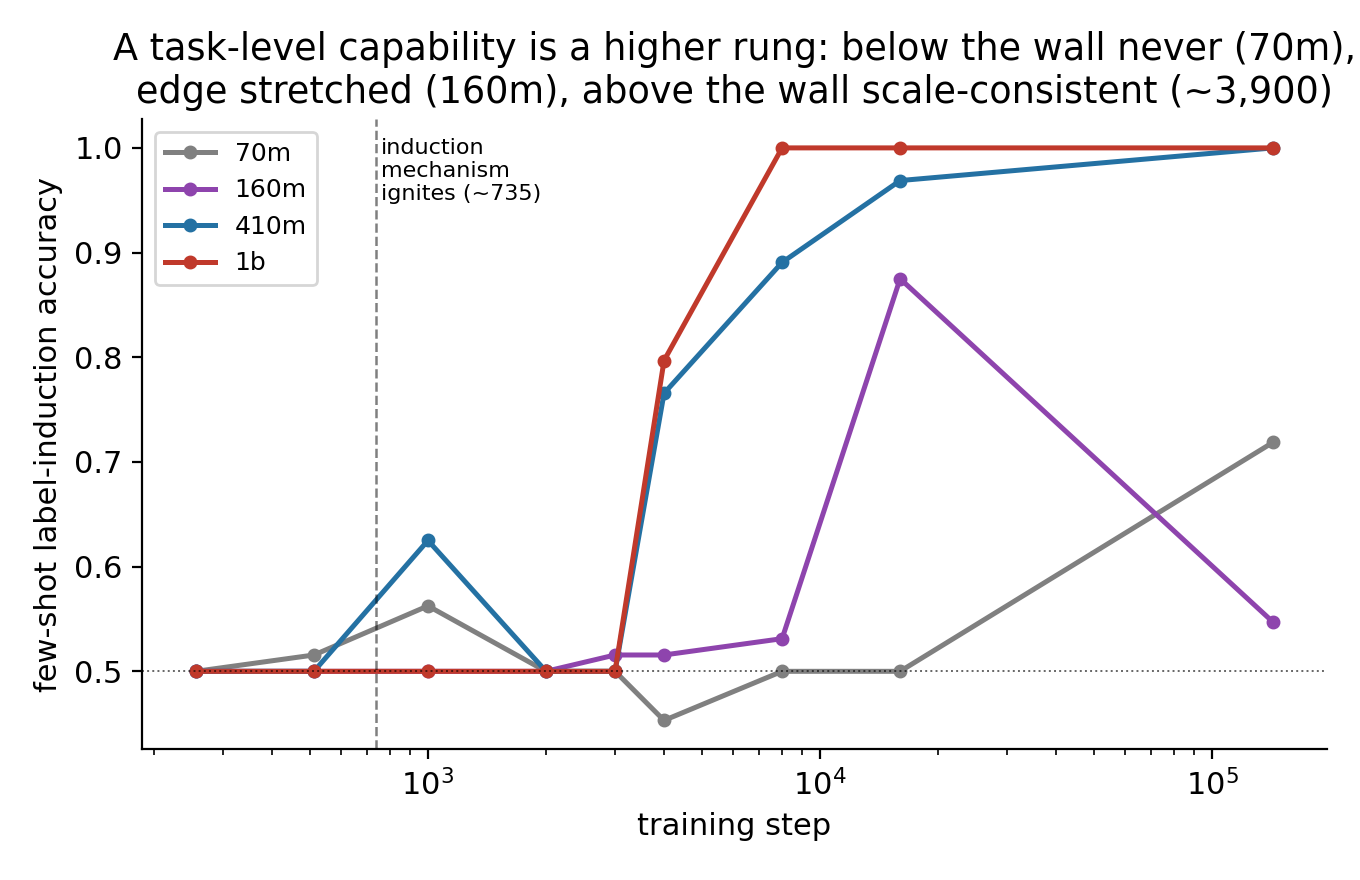}
\caption{\textbf{Task-level rung}: few-shot label induction never arrives below the wall, is stretched at its edge, and ignites scale-consistently above it.}
\label{fig:task}
\end{figure}

\subsection{The floor: a supersaturation threshold, proved intensive}

The floor is where the kinetic reading departs most sharply from the equilibrium one. \citet{gu2026mixing} document a data-mixing threshold and attribute it to capacity allocation --- an equilibrium answer (which phase wins the parameter budget). Equation~\eqref{eq:master} makes a stronger, kinetic claim: below a critical concentration the rate is \emph{suppressed beyond nucleation order} --- operationally zero for any subexponential token budget --- and the threshold is a property of the per-step stream composition, not of any total.

\begin{proposition}[The floor is a fold --- from structure, with no functional form assumed]
\label{prop:fold}
Let $\phi \in [0,1]$ be the alignment order parameter (the normalized product of attention masses on the $K$ required slots, so $\phi = 1$ iff every part is in place). Write the population drift as $\dot\phi = c\,G(\phi) - D(\phi)$, target minus competitor. Then the following three properties --- each \emph{derived}, none postulated --- suffice:
\textbf{(D1) the competitor's drift is linear in $\phi$}: $D(\phi) = r\phi$ with $r = K(1-c)\lambda' > 0$;
\textbf{(D2) the target's drift vanishes super-linearly at the origin}: $G$ is continuous on $[0,1]$ with $G(\phi)/\phi \to 0$ as $\phi \to 0^{+}$;
\textbf{(D3) it vanishes again at completion and is bounded between}: $G(1) = 0$, $G$ bounded, $G > 0$ on $(0,1)$.
Let $h(\phi) := G(\phi)/\phi$, extended by $h(0) := 0$; then $M := \max_{[0,1]} h$ is finite and attained at an interior $\phi^{*}$, and with $c_0 := r/M$:
\textbf{(a)} for $c < c_0$ the \emph{only} fixed point in $[0,1]$ is $\phi = 0$, and it is asymptotically stable --- the circuit direction has no resting place at all;
\textbf{(b)} for $c > c_0$ at least two positive fixed points $\phi_- < \phi_+$ exist, $\phi_-$ unstable and $\phi_+$ stable --- metastability with activated crossing over $\phi_-$;
\textbf{(c)} if $G \in C^1$ and $\phi^{*}$ is a non-degenerate maximum of $h$, the pair is born at $c = c_0$ through a tangency: a saddle--node (fold), not a transcritical exchange.
\end{proposition}

\begin{proof}
\emph{The three properties are consequences of results already established.}
(D1) is the replicator geometry of softmax attention plus the restoring-force computation: the exact identity $\partial \mathbb{E}[L]/\partial s_j = a_j(\ell_j - \langle\ell\rangle_a)$ (Lemma~\ref{lem:scorediff}) makes every drift on an attention mass \emph{proportional to that mass}, and the competitor supplies a per-unit-mass restoring force bounded below by $\lambda' > 0$ uniformly over the basin (Lemma~\ref{lem:restoring}), weighted by its stream share $1-c$. A product of $K$ masses each decaying at rate $(1-c)\lambda'$ decays at $K(1-c)\lambda'$, so linearity survives the product. Linear decay is thus forced by the simplex geometry, not chosen for convenience.
(D2) is no-partial-credit in dynamical form, and the two structures compose. By Lemma~\ref{lem:weakcoupling} the target's pull \emph{per unit mass} at a required slot $r$ is controlled by the product over the \emph{other} required slots, $\phi/A_r$; the replicator factor of (D1) then gives $\dot A_r \propto A_r \cdot (\phi/A_r) = \phi$, and $\dot\phi = \phi\sum_r \dot A_r / A_r$ yields, in the symmetric configuration $A_r = \phi^{1/K}$,
\[
G(\phi) \;\propto\; K\,\phi^{\,2 - 1/K},
\]
whose exponent exceeds $1$ for every $K \ge 2$, so $G(\phi)/\phi \to 0$. A term linear in $\phi$ would be precisely partial credit, which Lemma~\ref{lem:nopartial} excludes exactly; note also that the exponent \emph{rises} with $K$, so deeper conjunctions have flatter origins and, by the argument below, higher floors.
(D3) is the cross-entropy's own saturation together with the simplex bound: masses lie in $[0,1]$, so $G$ is bounded; and when every part is in place the model predicts the target correctly, the error probability $1 - p_c \to 0$, and with it the label-dependent gradient --- $G(1) = 0$. Nothing is imposed by hand: the gain rises out of the origin, is capped by the simplex, and is switched off by being right.
\emph{The conclusion follows.} $h$ is continuous on the compact $[0,1]$ (continuity at $0$ is (D2), at $1$ it is $0$ by (D3)), hence attains a finite maximum $M$ at some $\phi^{*}$, interior because $h > 0$ on $(0,1)$ and $h(0) = h(1) = 0$. Positive fixed points solve $c\,h(\phi) = r$. (a) If $c < c_0$ then $c\,h(\phi) \le cM < r$ everywhere, so no positive root; and near the origin $\dot\phi = \phi(c\,h(\phi) - r) < 0$, giving asymptotic stability of $\phi = 0$. (b) If $c > c_0$ then $c\,h(\phi^{*}) > r$ while $c\,h(0) = c\,h(1) = 0 < r$, so the intermediate value theorem gives a root on each side of $\phi^{*}$; the sign of $\dot\phi = \phi(c h - r)$ changes $-,+,-$ across them, which is exactly $\phi_-$ unstable, $\phi_+$ stable. (c) At $c = c_0$ the equation $c\,h = r$ holds only at $\phi^{*}$, where $h'(\phi^{*}) = 0$: the curves are tangent, and non-degeneracy ($h''(\phi^{*}) < 0$) gives the standard saddle--node normal form, so the two branches are born together rather than crossing.
\end{proof}

\begin{remark}
Earlier versions of this argument \emph{assumed} a reduced drift $-r\phi + cg\,\phi^m/(1+\phi^m)$ and flagged the linear decay and the saturating form as modeling choices. They are choices no longer. (D1) is the replicator geometry of the softmax --- every drift on an attention mass is proportional to that mass, exactly (Lemma~\ref{lem:scorediff}) --- combined with a restoring force per unit mass that is bounded below uniformly in the basin (Lemma~\ref{lem:restoring}); (D2) is Lemma~\ref{lem:nopartial} in dynamical form, a linear term being precisely the partial credit the task class excludes; (D3) is the simplex bound plus the fact that a model which is right receives no gradient. Nothing about the shape of $G$ between its two zeros is used --- \emph{no functional form is assumed at all} --- which also means the floor's existence cannot be an artifact of the algebra chosen to express it. \textbf{What the derivation costs, stated plainly:} Lemma~\ref{lem:scorediff} is general to softmax attention, but Lemmas~\ref{lem:restoring} and \ref{lem:weakcoupling} are proved inside the controlled gated-attention model of App.~\ref{app:occupation}, so (D1)--(D2) inherit that model's idealizations (positional attention, readout locality, oracle gating). The gain is therefore not that the floor has become assumption-free --- it is that the assumptions have moved from a \emph{chosen formula} to a \emph{stated model}, where each is a consequence of an established lemma and each carries its own measurement. A reader who rejects the gated model may still read Prop.~\ref{prop:fold} as before, taking (D1)--(D3) as three structural hypotheses; that is strictly weaker than the previous version, which additionally fixed the functional form. Each ingredient additionally has its own measurement: a floor appears only against a \emph{persistent} competitor (the three-diluent sweep below: inert and quickly-mastered fillers give no floor, so a fixed $r > 0$ is the operative condition), and the reinforcement demonstrably does not run away but balances at finite $\phi$ (below $c_0$ formed circuits dissolve to chance within 200 steps; above it they survive $8/8$). The model also predicts the phenomenology above the floor: bistability with activated crossing --- the memoryless waiting measured in EK-1 --- and a barrier that shrinks with drive --- the throttle. The floor hypothesis of the next proposition is thereby derived, not free-standing.
\end{remark}

\begin{proposition}[Intensivity and the effective floor (suppression beyond nucleation order)]
\label{prop:floor}
Suppose (i) the drive enters the generation term of Eq.~\eqref{eq:master} only through the per-step concentration $c$ of capability-relevant statistics (an intensive quantity), and (ii) for $c < c_0$ the circuit configuration is not metastable under the drift: it is disjoint from the attractor set, the destruction drift exceeding the generation drift along every formation path. Then:
\textbf{(a)} (intensivity) at fixed $c < c_0$, increasing total tokens, steps, or batch size cannot induce formation --- the floor is independent of budget;
\textbf{(b)} (beyond-order suppression) the per-attempt probability of a \emph{sustained} crossing ($\phi > 1/2$ held for a window $\tau$) is bounded by $e^{-a\tau/\eps}$ for some $a > 0$, so formation is suppressed beyond any nucleation-rate order: unlike escape into a metastable target, which is a one-shot barrier crossing, holding an unstable configuration costs action \emph{linearly in the holding time}. (We do not claim a strictly zero probability for a stochastic process; the claim is the qualitative separation between one-shot barrier crossing above the floor and time-extensive holding cost below it --- and hypothesis (ii) is not circular bookkeeping but the floor's measured content: sub-critical embryos are directly observed stalling and dissolving below $c_0$.)
\end{proposition}

\begin{proof}
(a) Under (i), the token budget enters only as the number of attempts $n$; the per-attempt success probability $p(c)$ depends on $c$ alone. If $p(c) = 0$ then $1 - (1-p(c))^n = 0$ for every $n$; more generally, by (b), $p(c) \le e^{-a\tau/\eps}$ uniformly in $n$, so any subexponential budget leaves the total probability vanishing as $\eps \to 0$.
(b) By hypothesis (ii) the set $\{\phi > 1/2\}$ contains no attractor of the drift for $c<c_0$. By the Freidlin--Wentzell estimates for occupation of neighborhoods of non-attracting compact sets \citep[Ch.~4]{freidlin2012random}, a trajectory that remains in such a neighborhood for time $\tau$ must follow a path whose action grows at least linearly in $\tau$ --- there is a per-unit-time cost $a>0$ to hold the state against the drift --- whence the probability bound $e^{-a\tau/\eps}$. (We use only the sojourn estimate, not the full exit theory. The drift hypothesis is no longer free-standing: Proposition~\ref{prop:fold} derives it from no-partial-credit plus saturation, and it is independently the measured content of the floor --- sub-critical embryos are directly observed stalling and dissolving below $c_0$.)
\end{proof}

The measured floor matches both clauses. Intensivity: a $4\times4$ concentration$\times$batch grid (Fig.~\ref{fig:floor}) shows $c=0.20$ never igniting at any batch (b64 = 8$\times$ the tokens of b8) while $c \ge 0.35$ always ignites; a finer sweep locates the 70M floor sharply (0/12 for $c\le0.26$; 4/5 at $c=0.28$). Beyond-order suppression (not mere slowness): the near-floor divergence $\tau \propto (c - c_0)^{-\gamma}$, $\gamma = 0.62\pm0.01$ ($R^2$ 0.999), with sub-critical embryos observed stalling and dissolving below the floor --- the driving-force law of nucleation, not a capacity story. The floor is intensive but \emph{not} scale-invariant: it falls $0.27 \to 0.20 \to {\sim}0.17$ from 70M to 410M, consistent with a weak power law, $c_0 \approx 0.26\,(N/70\mathrm{M})^{-0.26}$ ($R^2 = 0.96$ on three points --- reported as the measured trend and an extrapolation hypothesis, not an established law, since three points cannot fix an exponent). Within the igniting region, batch enters as $b^{-0.33\pm0.05}$ and learning rate only accelerates: \emph{at fixed mix, lr and batch are clocks; the mix is the drive}. Above the floor the throttle is steep: never / $\sim$1{,}400 / $\sim$470 steps at doses 0.15/0.30/0.50 (Fig.~\ref{fig:throttle}).

\textbf{What ``concentration'' has to mean: three diluents, one apparatus.} The real-model floor leaves a definitional question open --- concentration of \emph{what}, measured against \emph{what}? We settled it in the shortcut-free apparatus by holding the target task, the model, and the optimizer fixed and changing only the material the target is diluted \emph{into}, sweeping the target's share over up to $33\times$ in each case. \emph{Inert diluent} (uniform random tokens): formation time is flat --- log--log slope $0$ over an $8\times$ density change, no floor. The inertness is the point: unpredictable filler contributes no gradient, so the target's share of the \emph{gradient} never changes however its share of the \emph{tokens} falls; and a pure gain change is in any case normalized away by AdamW's per-parameter scaling. \emph{Saturating diluent} (a single competing copy skill): slope $-0.22$, still no floor, and at intermediate dilution formation is \emph{faster} than at full strength --- the competitor is learned within $\sim$100 steps, after which it stops competing, and the offset-selective head it built seeds the target (the cross-nucleation effect of \citet{lei2026circuits}, reproduced unbidden). \emph{Non-saturating diluent} (a family of $24 > H$ copy skills, so the background always has something left to learn): the slope steepens past the product law and diverges ($-0.55 \to -1.67$), and a floor appears. The lesson is a definition: \textbf{supersaturation is the target's share of the gradient against a persistent competitor, not its share of the token stream} --- which is exactly why natural text, never exhausted, produces a floor in real models while an inert or quickly-mastered filler cannot.

\textbf{The floor is a stability threshold, not an absence of events.} Pushing the target share down in the non-saturating apparatus, formation events do not simply stop: at $0.5\%$ the waiting times split into two clusters ($3.8$--$6.9$k versus $83$--$159$k steps), one seed forms and then \emph{dissolves}, and one is censored at $2\times10^{5}$ steps --- a flickering band, the signature of net rate crossing zero rather than of a hard kinetic wall. Measuring the destruction side directly resolves it. Seeding the circuit at $c = 0.25$ and then switching the stream to a maintenance concentration $c_{\rm hold}$ gives survival $8/8$, $8/8$, $8/8$, $5/8$, $2/8$, $0/8$ at $c_{\rm hold} = 0.25, 0.06, 0.02, 0.005, 0.001, 0$ (30k held-out steps; $2\%$ vs $0.1\%$ Fisher $p = 0.007$, vs $0$ $p = 1.6\times10^{-4}$), with the un-reduced arm surviving $8/8$ --- so the collapse is caused by the concentration drop, not by continued training. At zero drive a formed circuit dissolves to chance within 200 steps. The two measurements agree on where the floor sits ($0.1$--$0.5\%$), and together they say something sharper than the proposition assumed: \emph{below the floor the capability is not merely never built, it cannot be held} --- Proposition~\ref{prop:fold}'s ``no resting place'' read as an experiment. The two protocols are not identical, and this yields a prediction for real models: below their floor, budgets far longer than ours should show transient, flickering formation rather than a blank, since our $0/12$ was measured at a fixed budget.

\begin{figure}[htbp]\centering
\includegraphics[width=0.7\linewidth]{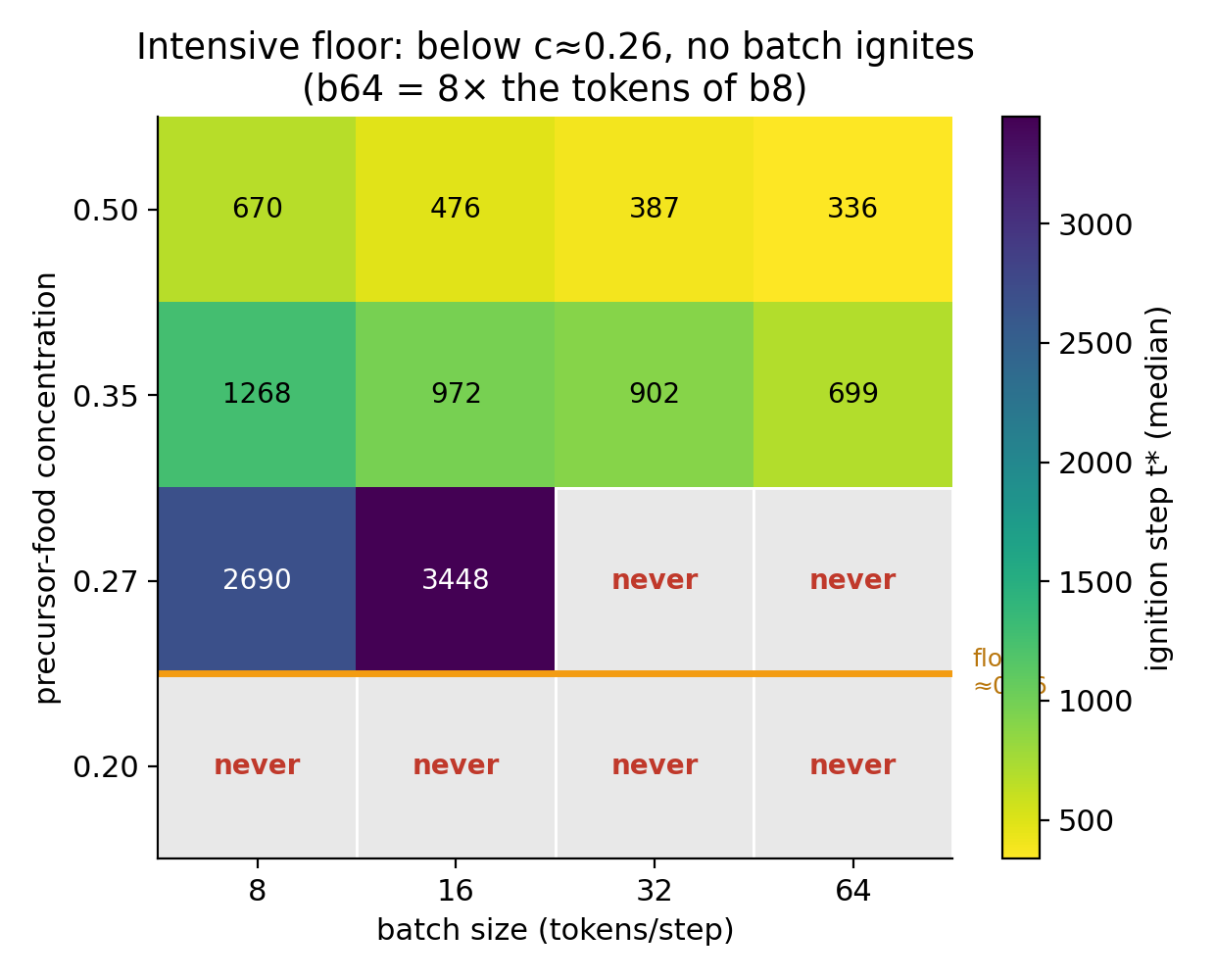}
\caption{\textbf{The intensive floor}: below $c\approx0.26$ (70M) no batch ignites --- token volume cannot substitute for concentration.}
\label{fig:floor}
\end{figure}

\begin{figure}[htbp]\centering
\includegraphics[width=0.62\linewidth]{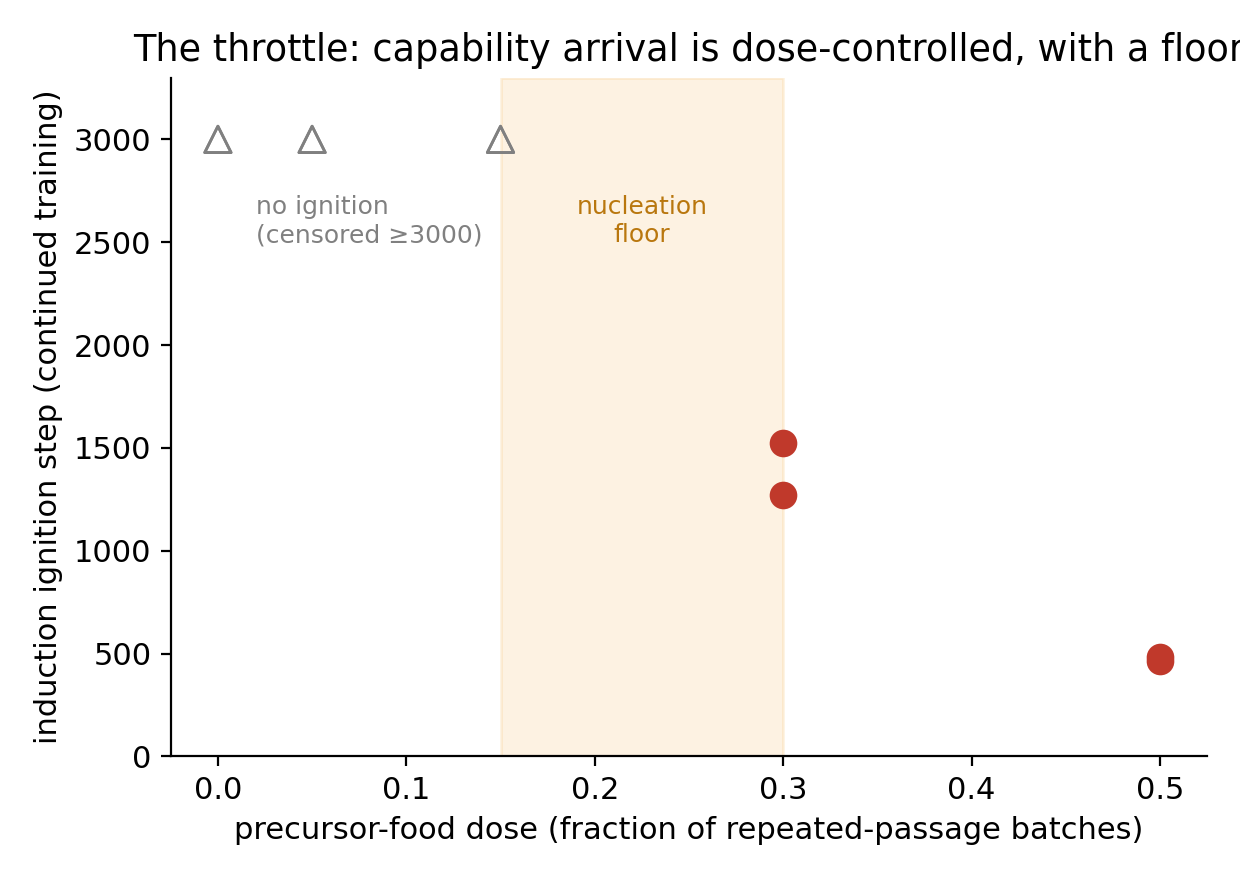}
\caption{\textbf{The throttle}: dose-controlled arrival with a nucleation floor.}
\label{fig:throttle}
\end{figure}

\textbf{The switch, including from scratch.} Ignition follows an experimenter-chosen switch of the data mix past the floor 10/10; without it, never. From random initialization --- no pretrained substrate --- the same three laws reproduce: plain data never ignites (0.02--0.03 through 16k steps); the switch ignites 2/2 ($\sim$700 steps); a delayed switch pays a growing tax; switching too late never ignites (Fig.~\ref{fig:scratch}; the tax and deadline belong to the backward reading, \S\ref{sec:plasticity}). Cross-family, on OLMo-1B: control 0/2 never, switch 2/2 at $\sim$247, delayed switch taxed ($\tau$ 471/799).

\begin{figure}[htbp]\centering
\includegraphics[width=0.9\linewidth]{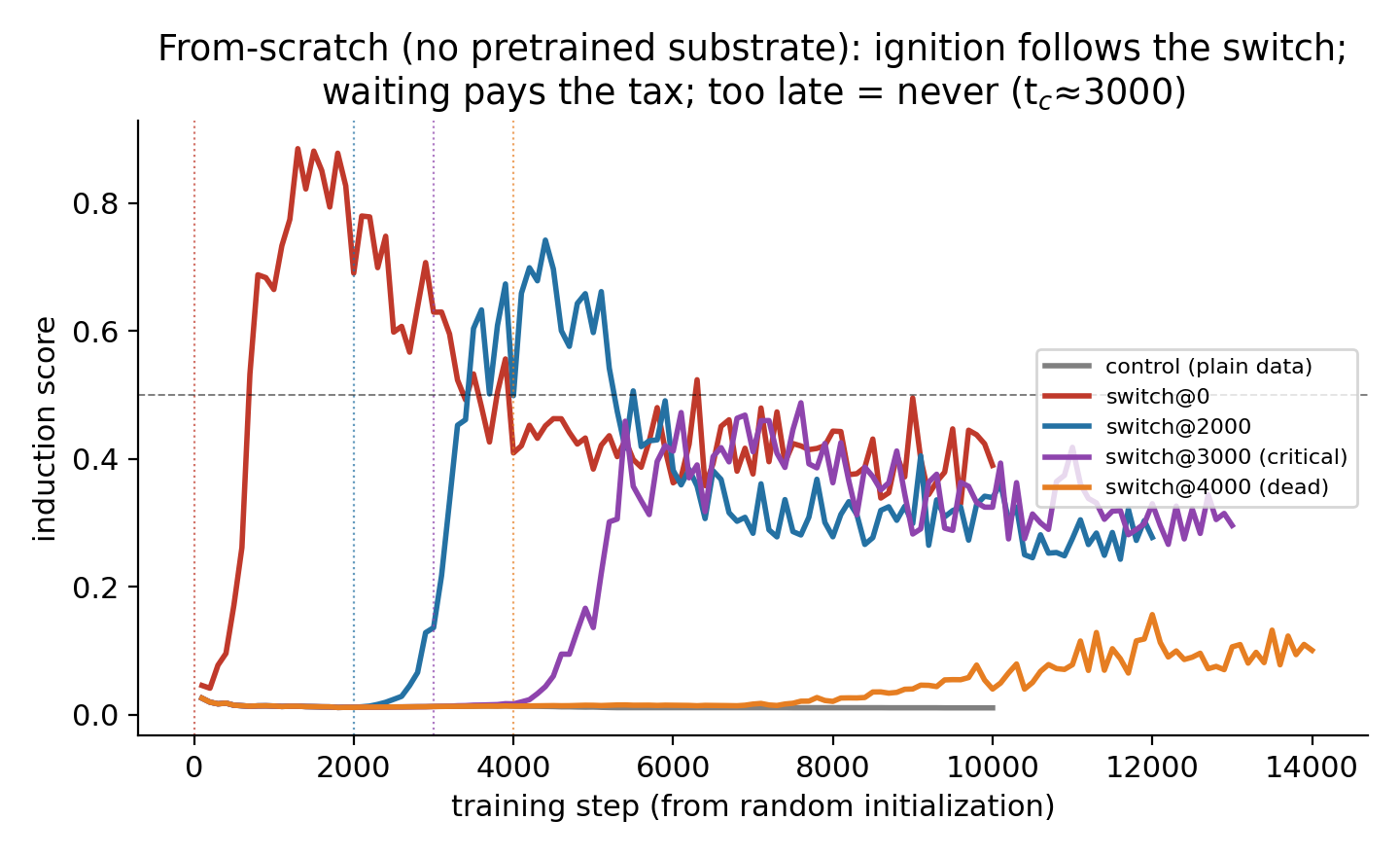}
\caption{\textbf{From random initialization}: ignition follows the switch; waiting pays the occupation tax; too late never ignites ($t_c\approx3000$).}
\label{fig:scratch}
\end{figure}

\section{Reading the equation backward: loss of plasticity}
\label{sec:plasticity}

Backward means: hold the drive available but let training itself consume the \emph{site} term $N$. The equation then predicts a specific, testable structure for how learnability dies --- and a discriminator separating it from every drive-side explanation.

\begin{proposition}[Site-limited versus drive-limited failure, and the horizon]
\label{prop:sites}
In Eq.~\eqref{eq:master} with factorized generation $G = N \nu\, \sigma(c)\, e^{-\beta K}$:
\textbf{(a)} $\partial G/\partial c = N \nu\, \sigma'(c)\, e^{-\beta K} > 0$ whenever $N > 0$ and $\sigma' > 0$, while if $N = 0$ then $G = 0$ \emph{for every} $c$. Hence raising the drive rescues a drive-limited substrate and cannot rescue a site-limited one: the response to a drive increase is a discriminating experiment between the two failure modes.
\textbf{(b)} Let $N_{\min}(c) := D / (\nu\,\sigma(c)\,e^{-\beta K})$ be the fewest sites at which generation clears destruction at drive $c$; it is decreasing in $c$. If the site count decays with occupation time $S$ (training on other data) as a non-increasing $N(S)$, the \emph{plasticity horizon at drive $c$} is $t_c(c) = \inf\{S : N(S) < N_{\min}(c)\}$, non-decreasing in $c$. Past $t_c(c)$, formation is impossible at that drive, and possible again only by one of exactly two routes: re-opening sites (raising $N$), or raising the drive past the \emph{escape threshold} $c_{\rm esc}(S) := \inf\{c : N_{\min}(c) \le N(S)\}$ --- the drive at which generation on the residual sites clears destruction again.
\end{proposition}

\begin{proof}
(a) is immediate from the factorized form (itself the mean-field structure whose validity domain is measured in \S\ref{sec:synthesis}): $c$ enters only through $\sigma(c)$, multiplied by $N$. (b) $N_{\min}(c)$ is decreasing because $\sigma$ is increasing; hence $t_c(c)$ is non-decreasing. For $S > t_c(c)$, $N(S) < N_{\min}(c)$ gives $G = N(S)\nu\sigma(c)e^{-\beta K} < D$: net rate negative, $\phi$ cannot cross. The net rate is restored exactly by raising $N$ above $N_{\min}(c)$ or by raising $c$ until $N_{\min}(c) \le N(S)$, i.e.\ $c \ge c_{\rm esc}(S)$; monotonicity of $N_{\min}$ makes $c_{\rm esc}$ a sharp threshold.
\end{proof}

\textbf{Site dynamics, closing the loop.} The forward reading has a rate law; symmetry demands the backward reading have one too, not an exogenous ``$N$ decreases.'' The closure is
\[
\frac{dN}{dS} \;=\; -\,\lambda(\theta(S))\, N \;+\; R,
\]
where $\lambda$ is the occupation rate --- mechanistically identified below as attention-subspace commitment --- and $R$ is site re-opening, zero in native training and \emph{exogenously actuated} by our reset and graft interventions (they are the $R$ term, performed by hand). We deliberately do not impose a parametric form for $\lambda$: with the occupancy protocol run at zero target drive ($c = 0$ during aging, so the depletion measured here is occupation by the base task alone, not consumption by target-crystal growth --- the coupled case is outside this protocol and stated as such), the equation integrates to $N(S)/N(0) = e^{-\int_0^S \lambda}$, and since formation time scales as $1/N$ at fixed drive (Prop.~\ref{prop:sites}), the \emph{measured, non-parametric} depletion curve is $N(S)/N(0) = \tau(0)/\tau(S)$: $1 \to 0.46 \to 0.45$ at $S = 0/600/1200$, crossing $N_{\min}(c{=}0.5)$ before $S = 3600$. Three points and a bound do not fix a functional form, and we decline to overfit one; what they do fix is the \emph{structure} --- depletion is real, monotone, drive-independent in this protocol, mechanistically located ($\lambda$ tracks commitment at $r = 0.93$), and reversible only through the $R$ term. The scale law enters as $d\lambda/d(\text{size}) < 0$ (aging slope $0.69 \to 0.16$ from 70M to 410M).

The proposition's content is that these are \emph{measurable alternatives}, and every measurement lands on the site side. \textbf{The discriminator}: on a substrate aged past its horizon, doubling the drive does not rescue (0/3), while \emph{opening new sites does} --- grafted fresh head pairs ignite with lightning kinetics (200--500 steps) at per-pair success $p \approx 0.15$ consistent with independent-site statistics, and sites must be opened \emph{in pairs} across both layers (single-layer openings: 0/6), exactly as a two-layer composition requires. Learnability death is a site-term failure, not a drive-term failure. \textbf{The tax and the horizon}: the post-switch delay grows with waiting ($n=5$ seeds: $\tau = 506\pm44 \to 1108\pm66 \to 1115\pm65$ at $S = 0/600/1200$); at $S = 3600$ no ignition in 4{,}400 further steps (tail slope zero --- not censoring). The horizon has a scale law: $t_c \approx 3200$ (70M), $>6000$ (160M), $>8000$ (410M); the aging slope falls $0.69 \to 0.16$ --- larger substrates age slower (Fig.~\ref{fig:shelf}). \textbf{Drive-relative death}: the dead substrate revives above a sharp escape drive --- 0/2 at 0.60--0.69, 2/3 at 0.71, 2/2 at $\ge$0.72; threshold $c_{\rm esc} = 0.705\pm0.005$ --- clause (b)'s second route, observed, including its predicted \emph{sharpness} (monotone $N_{\min}$ makes $c_{\rm esc}$ a threshold, not a gradient). \textbf{The located mechanism}: during waiting, every head's attention pattern grows more committed (peakedness $+45\%$, entropy $-20\%$), tracking the tax at $r = 0.93$. What ages is \emph{recruitable capacity}, not the existing parts: the precursor head itself keeps growing throughout the waiting period, and what is exhausted is the supply of uncommitted heads --- which is why the failure is site-limited (above) and why it is repaired by re-opening sites rather than by rebuilding parts. Targeted resets make this causal --- restoring only the \emph{attention} weights of a dead model rescues ignition 5/5 (vs.\ 0/5 un-reset controls), restoring MLP rescues weakly, restoring embeddings does not (0/2), replicating at 160M and 410M. The two reset arms separate with non-overlapping Wilson 95\% intervals ($5/5$: $[0.57, 1.00]$; $0/5$: $[0.00, 0.43]$; Fisher exact $p = 0.008$), as do the floor arms ($0/12$ below the floor: $[0.00, 0.24]$; $4/5$ above it: $[0.38, 0.96]$). \textbf{The active ingredient, isolated.} Theorem~\ref{thm:occupation}(iii) says occupation lives in the attention \emph{pattern}, so the score channel should carry the entire cure and the value channel none of it; the apparatus confirmed this (App.~\ref{app:occ-measured}), and the prediction was then tested on the real substrate by slicing the fused QKV projection head-wise. It separates completely (six seeds, one aged substrate per seed, all four arms run from it): un-reset $0/6$; re-initializing only the \emph{query--key} slices $6/6$ (Fisher $p = 0.002$ against control); re-initializing only the \emph{value} slices and the output projection $0/6$ --- indistinguishable from doing nothing (induction score $0.252$ vs $0.263$ un-reset, against $0.700$ for the QK arm); and the whole-block reset $6/6$, exactly equal to the QK arm. \emph{The block reset's entire therapeutic effect comes from the score channel.} This converts the repair from a coarse intervention into a prescription --- reset QK, keep OV --- and supplies the sharpest causal evidence in the paper that what ages is attention \emph{allocation} rather than the content the heads carry. (Note the distinction: this reset \emph{re-opens sites} for a capability that has never formed on this substrate; it does not act on formed structure. The melt hierarchy of \S\ref{sec:control}, in which formed compositions dissolve before their parts, is a different axis --- destruction of what exists versus recovery of the ability to build --- and the two results are consistent, not in tension.) \textbf{Loss-blindness}: throughout the aging that kills inducibility, validation loss falls smoothly ($5.46 \to 1.63$) --- standard monitoring is structurally blind to the site term, which is why the phenomenon has stayed invisible to ordinary telemetry (Fig.~\ref{fig:lossblind}).

\begin{figure}[htbp]\centering
\includegraphics[width=0.75\linewidth]{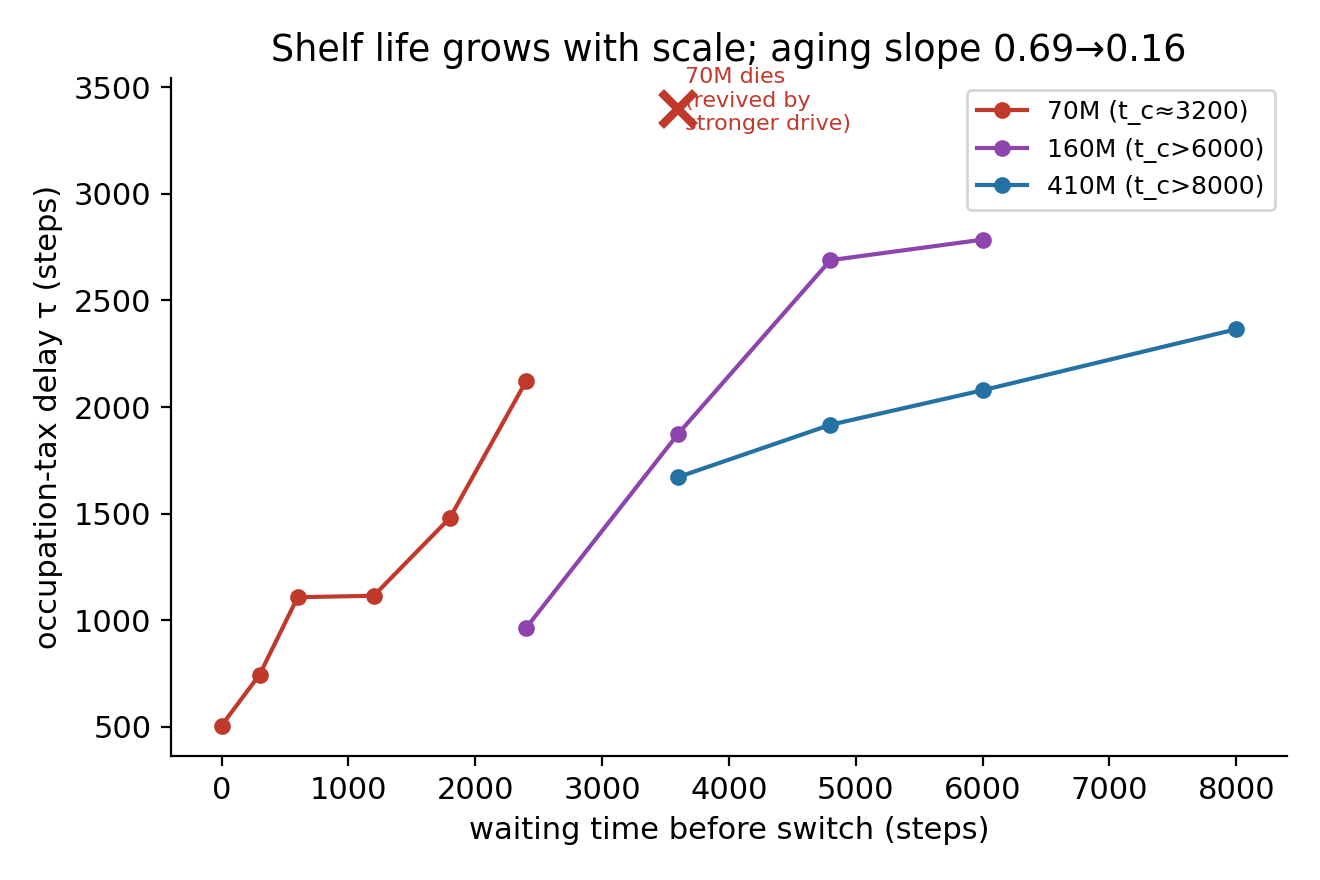}
\caption{\textbf{Shelf life grows with scale}; the occupation tax accumulates $\sim$4$\times$ slower at 410M than 70M.}
\label{fig:shelf}
\end{figure}

\begin{figure}[htbp]\centering
\includegraphics[width=0.75\linewidth]{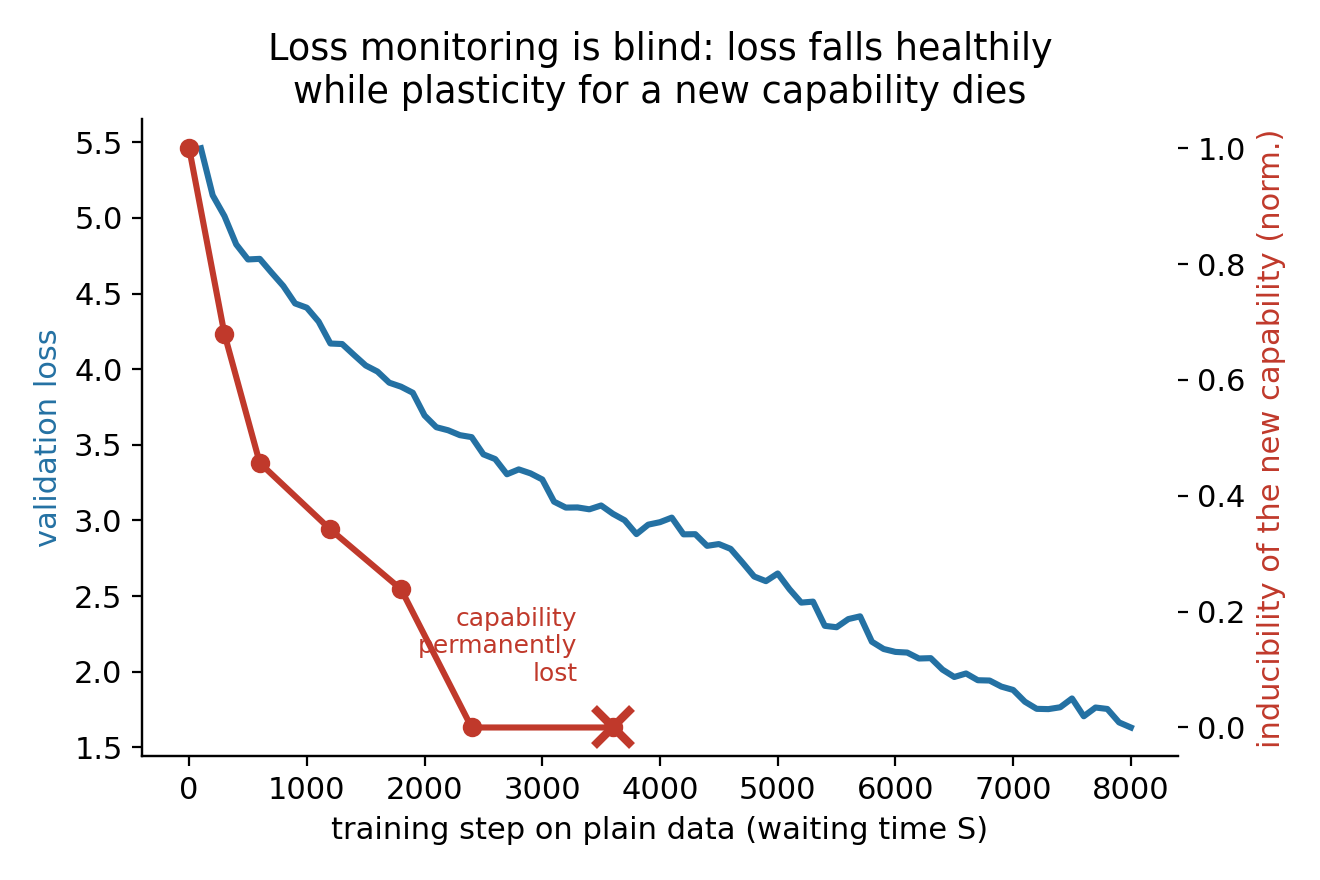}
\caption{\textbf{Loss monitoring is blind}: validation loss falls healthily during exactly the waiting period in which the capability becomes permanently uninducible at fixed drive.}
\label{fig:lossblind}
\end{figure}

\subsection{The deadline as a theorem: occupation in a controlled gated-attention model}
\label{sec:occupation}

Everything above is measured. This subsection reports that the site story is also \emph{provable} --- as a theorem in a controlled model --- with a division of labor that deserves emphasis: the deadline's \emph{consequences} (irrecoverability by data, repair by reset, invisibility in the loss) require \textbf{no ergodicity assumption at all}. Confinement is the easy direction of stochastic analysis --- a drift inequality plus bounded increments (Hajek's drift theorem \citep{hajek1982hitting}) --- in contrast to the rate law's upper-bound side, which carries Assumption~\ref{as:mixing}. The paper's backward reading thus stands on the unconditional side of the mathematics.

\emph{The model} (full detail in App.~\ref{app:occupation}): $H$ position-attention heads (softmax scores over $W$ window offsets); a data stream mixing \emph{background skills} (cue token $g_i \to$ copy the token at offset $o_i$; traffic share $\pi_i$), \emph{target queries} (share $c$; the $K$-slot no-partial-credit task of Lemma~\ref{lem:nopartial}), and neutral substrate; a readout \emph{gated by the current token}, with the substrate channel closed. \textbf{The gate is the model's most idealized ingredient --- it hands the model an oracle task identity.} The point of retaining it is sharp: in this model the difficulty of late learning is provably \emph{pure capacity occupation} --- even knowing exactly which task is which, a substrate short of $K$ uncommitted heads cannot form the capability. Each structural assumption has the empirical counterpart measured above (commitment at $r=0.93$; control ablations retaining 99\% --- the readout does approximate ``read only the carrier''; reset 5/5), which licenses the abstraction without claiming it as a theorem.

\begin{theorem}[Occupation: the silent deadline, in the gated model]
\label{thm:occupation}
Under assumptions (G1)--(G8) and feasibility/scale conditions C0--C4 of App.~\ref{app:occupation}, with formation requiring $K$ heads on distinct target slots:
\textbf{(i) (deadline; uses the recruitment assumption)} the number of uncommitted-or-loosely-pinned heads is non-increasing with high probability over the recruitment window, and the deadline $t_{\rm dead}$ (fewer than $K$ such heads) is a.s.\ finite, with $\mathbb{E}[t_{\rm dead}] \lesssim (H-K+1)/(\rho_{\min} p_0)$. \emph{We do not claim a constant per-head recruitment hazard}: the traffic-share structure of C0 ($\mu_i \propto \pi_i$) implies the opposite, and we measure it below.
\textbf{(ii) (data cannot rescue; no mixing)} past $t_{\rm dead}$, $P(\text{formation within } T \text{ steps}) \le HKT\, e^{-\theta \Delta}$ with $\theta = \Omega\!\big(\mu/(\eta G^2)\big)$; the target share $c$ enters $\mu$ only linearly, so rescue by data requires $c$ large enough to break the background's own pinning margin --- the catastrophic-forgetting dual, not a gentle scaling.
\textbf{(iii) (reset rescues; no mixing)} re-initializing the attention scores of any $K$ heads restores the fresh-substrate formation rate.
\textbf{(iv) (silence; no mixing)} pre-formation the target loss equals $\ln V$ plus a bounded junk term held near zero, and the background loss decreases through $t_{\rm dead}$ in steps of the same order at every pinning event: the deadline produces \emph{no discontinuity, kink, or necessary precursor in any loss-derived scalar}. The step that kills learnability looks, in the loss, like every other step of ordinary background learning.
\end{theorem}

\emph{Proof chain} (full proofs in App.~\ref{app:occupation}): pre-formation, the target exerts on any single head only a pull proportional to a product of attention masses on the required slots --- an iterated-oscillation bound on full-support Fourier coefficients, valid because the readout sees window contents only through head outputs --- together with an exact decomposition of the cross-entropy gradient into a \emph{uniformizing potential} (label-free) and a signal covariance; the attention simplex makes carried and target patterns exclusive; a linear-response computation, every block carrying the $(1-p_c)$ factor, gives a restoring drift toward the carried offset; an exact score-difference identity turns the force ledger into a drift inequality; Hajek's theorem converts drift into exponential confinement with \emph{no ergodicity input}; and stochastic domination of the recruitment process yields the deadline.

\textbf{The theorem's predictions, tested --- including one we had to withdraw.} We ran the theorem's own falsifiable consequences in the apparatus (App.~\ref{app:occ-measured}). \emph{Confirmed, and sharper than the real-model result that motivated it}: the repair is specific to the attention \emph{score} channel. From one aged substrate, re-initializing the QK slices recovers formation 4.3$\times$ (6/6 formed), MLP does little ($0.81\times$), and re-initializing the OV slices is \emph{worse than no reset at all} ($1.76\times$ slower) --- occupation lives in the attention pattern, so freeing the pattern is the repair, while resetting the value channel destroys learned structure and frees nothing. This turns \S\ref{sec:plasticity}'s block reset (5/5 vs 0/5) into a prescription --- reset QK, keep OV --- which we then confirmed prospectively on the real substrate: on aged Pythia-70M the QK arm ignites 6/6, the OV arm 0/6 (identical to no reset), and the block reset 6/6, i.e.\ exactly the QK arm and nothing more (\S\ref{sec:plasticity}). \emph{Falsified}: we had additionally assumed a constant per-head recruitment hazard, which would make occupation decay exponentially and the deadline scale as $\log H$. It does not: the hazard falls $\approx$15$\times$ across the sweep ($32 \to 18 \to 8 \to 5$ available slots), and extrapolating the initial rate misses by $\sim$1500$\times$. The failed clause was an extra simplification that \emph{contradicted our own traffic-share model} --- with pull $\mu_i \propto \pi_i$, high-traffic skills take heads first and the remainder recruit slowly, so the hazard must decline. We withdraw the clause and the $\log H$ prediction with it; the mechanism, the finiteness of the deadline, and clauses (ii)--(iv) are untouched. The correction also resolves a tension we had flagged: Pythia's saturating depletion profile ($1 \to 0.46 \to 0.45$) is the same declining-hazard family, not a counterexample.

\textbf{Position in the plasticity literature.} The phenomenon is established \citep{dohare2024loss,ash2020warm}, the remedies are empirical \citep{sokar2023dormant}, and the causes literature is explicitly partial --- no single proposed mechanism suffices \citep{lyle2024disentangling}. \citet{lyle2026grokking} conjecture, via effective learning rate, that the dynamics of emergence and of plasticity loss are linked; we cite that as the conceptual precursor and supply the quantitative form the conjecture lacks: the \emph{same} rate equation whose forward reading is \S\ref{sec:emergence}, with a law ($\tau(S)$), a horizon ($t_c$ and its scale law), a located and causally verified mechanism (attention-subspace commitment), a repair (targeted reset), and a discriminator (drive versus sites) that rules out the drive-side alternatives. Theorem~\ref{thm:occupation} adds the tier the causes literature is missing: a proof, in a controlled model, that occupation \emph{must} produce exactly this phenomenology --- deadline, data-irrecoverability, reset-repair, and silence.

\section{Completing the equation: injected temperature and process control}
\label{sec:control}

The knockdown component of $D$ in Eq.~\eqref{eq:master} is dormant in native training (\S\ref{sec:notemp}: nothing qualifies as a temperature, and at practical learning rates weight drift, not noise, is the only destroyer of formed structure). Installing the missing term --- an isotropic weight-noise bath of magnitude $T$, a few lines of code --- completes the equation and turns it into a control surface. Everything in this section is measured in a controlled apparatus (a conjunction-circuit toy with $K = 2..4$, plus real-Pythia induction circuits), with constants frozen and out-of-sample tests preregistered.

\paragraph{Why inject a temperature at all?}
Since emergence needs no bath (\S\ref{sec:emergence}: the drive is data), one may ask whether this section adds anything data cannot. It adds exactly the operations that a \emph{monotone} drive cannot express. Data can only push formation forward, in barrier order; it cannot \emph{pause} a capability below threshold and release it on command (pinning at $\phi_{ss}$), cannot \emph{selectively erase} one circuit while sparing its parts (the melt line --- the primitive machine unlearning lacks), and cannot \emph{accelerate} a high-barrier circuit past its own kinetics (the nose, where annealing beats a gradient-filter baseline). The bath is also the theory's \emph{instrument}: melt points, the hysteresis window, dose--damage bookkeeping, and $\phi_{ss}(T)$ are measurements of the destruction term and of basin depths that native training simply cannot perform --- no thermometer, no thermometry. We therefore present these schedules as existence proofs that every term of Eq.~\eqref{eq:master} is independently actuatable --- the completion of the program --- and as measurement instruments, not as a claim that production training should run with a bath on.

\subsection{The nose: barrier-gated annealing}

With a qualified temperature, generation gains a barrier-assistance factor (crossing aided by total noise $\eps_0 + T$) while destruction turns on above a damage onset. These two monotonicities force the shape of the formation-time curve:

\begin{proposition}[Existence of a TTT nose, if and only if barrier-gated]
\label{prop:nose}
Let the generation rate $g(T)$ be continuous and strictly increasing in $T$ (barrier assistance), and the destruction rate $d(T)$ be continuous, non-decreasing, $d \equiv 0$ on $[0, T_{\rm dam}]$ for some damage onset $T_{\rm dam} > 0$, with $g(T_{\rm ceil}) = d(T_{\rm ceil})$ at some finite ceiling $T_{\rm ceil} > T_{\rm dam}$. Define the formation time $\tau(T) = C/(g(T) - d(T))$ on $[0, T_{\rm ceil})$. Then:
\textbf{(a)} $\tau$ is strictly decreasing on $[0, T_{\rm dam}]$ and $\tau(T) \to \infty$ as $T \to T_{\rm ceil}^{-}$; hence $\tau$ attains an interior minimum (a \emph{nose}) in $[T_{\rm dam}, T_{\rm ceil})$.
\textbf{(b)} If instead the barrier is negligible --- $g$ constant in $T$ --- then $\tau$ is non-decreasing on all of $[0, T_{\rm ceil})$ and the minimum is at $T = 0$: cold is fastest, and no nose exists.
Thus an interior nose exists precisely when there is a barrier for temperature to assist.
\end{proposition}

\begin{proof}
(a) On $[0, T_{\rm dam}]$, $d = 0$ and $g$ is strictly increasing, so $\tau = C/g$ is strictly decreasing there; and $g - d \to 0^{+}$ as $T \to T_{\rm ceil}^{-}$, so $\tau \to \infty$. Since $\tau$ is continuous on $[T_{\rm dam}, T_{\rm ceil})$ and diverges at the right end, it attains a minimum on that interval at some finite $T^{*}$; by the strict decrease on $[0, T_{\rm dam}]$, $T^{*}$ is the global minimizer over $[0, T_{\rm ceil})$ and satisfies $T_{\rm dam} \le T^{*} < T_{\rm ceil}$ --- an interior point of the temperature domain, i.e.\ a nose. (b) With $g$ constant, $g - d$ is non-increasing, so $\tau$ is non-decreasing; the minimum is at $T = 0$.
\end{proof}

The measured curves realize both branches exactly as the proposition sorts them: $K = 2$ (barrier assist capped at $1.11\times$) shows no nose in six curves --- cold fastest, the (b) branch; $K = 3$ (assist $2.17\times$) shows a clean nose at $T \approx 0.003$, non-overlapping across seeds, depth 23--41\%; $K = 4$ confirms. The full C-curve formula (attempt $\times$ assistance $\times$ ceiling proximity), with constants frozen on-axis, predicts a held-out drive lane at median 13.9\% error and a prospective double-nose position hit at 14.3\% --- and yields the intrinsic noise scale of SGD itself, $\eps_0 \approx 5\times10^{-4}$, as a fit constant (the quantitative version of ``the cold arm does not freeze''). As an application, nose-temperature annealing accelerates a chosen high-barrier circuit $\sim$2$\times$ over the Grokfast baseline \citep{lee2024grokfast} --- a mechanism-guided schedule versus a gradient filter. And the construction is a genuine TTT diagram, not only a formation-time curve: post-nucleation growth follows Johnson--Mehl--Avrami--Kolmogorov kinetics with measured Avrami exponent $n \approx 1.16$ (single-front growth) \citep{avrami1939kinetics,lei2026circuits}, supplying the transformed-fraction axis $X(t,T)$ that the metallurgical analogy requires.

\subsection{The melt line: destruction is a weakest-link law}

Formation and destruction read the same part count $K$ through opposite logical gates, and this duality is provable.

\begin{proposition}[Weakest-link melting, and the build/break duality]
\label{prop:melt}
Let a circuit function if and only if all $K$ of its links function (the defining conjunction), and let link $i$ fail when the perturbation exceeds a threshold $\sigma_i$. Then:
\textbf{(a)} the circuit's melt threshold is $\sigma^{*} = \min_i \sigma_i$: it fails as soon as any link fails;
\textbf{(b)} for nested circuits $S \subset S'$, $\sigma^{*}(S') \le \sigma^{*}(S)$: composition depth can only lower, never raise, the melt point;
\textbf{(c)} (duality) over the same part set, the formation cost is a \emph{sum} (Prop.~\ref{prop:additivity}: all parts must co-occur --- an AND for building) while the destruction threshold is a \emph{min} (any part's failure suffices --- an OR for breaking). Compositional circuits are therefore simultaneously hard to build and easy to break, with both facts derived from the same $K$.
\end{proposition}

\begin{proof}
(a) Failure of the circuit is the union $\bigcup_i \{\text{link } i \text{ fails}\}$, which first occurs at perturbation $\min_i \sigma_i$. (b) The minimum over a superset is at most the minimum over the subset. (c) is (a) juxtaposed with Prop.~\ref{prop:additivity}: the intersection event prices formation as a sum of exponents; the union event prices destruction as a minimum of thresholds.
\end{proof}

Measured, at circuit grain and across scale: melt thresholds under injected weight noise order parts $\gg$ compositions in 20/21 model--pair tests across both families (2.5--6$\times$; one 8\% local inversion reported as such; Fig.~\ref{fig:lifo}); on real Pythia the previous-token head melts at 5.3--6.4$\times$ the induction head's threshold at every scale --- both are near scale-invariant \emph{material constants} (induction: 0.0017--0.0025 across 31M--410M synthetic; 0.00028--0.00063 real-text), and their ratio is a constant of the composition. The parts survive to perturbations where the composition's inter-layer \emph{wiring} has long failed --- the weakest link is the composition itself, clause (b) in the wild. Melting is first-order-like (cliff, not slope), sits above the formation ceiling (a hysteresis window, measured (0.03, 0.05) in the toy), and melt points scale with maintenance strength (3--8$\times$ between data regimes) --- a dynamic, driven melting point, as the non-equilibrium branch requires. Applications: \emph{selective dissolution} --- at $T = 0.001$, induction dies while the previous-token part survives, demonstrated on real Pythia at three scales; and a \emph{risk rule} for practitioners --- assess a model's fragility by its longest chain, not its parts.

\begin{figure}[htbp]\centering
\includegraphics[width=0.85\linewidth]{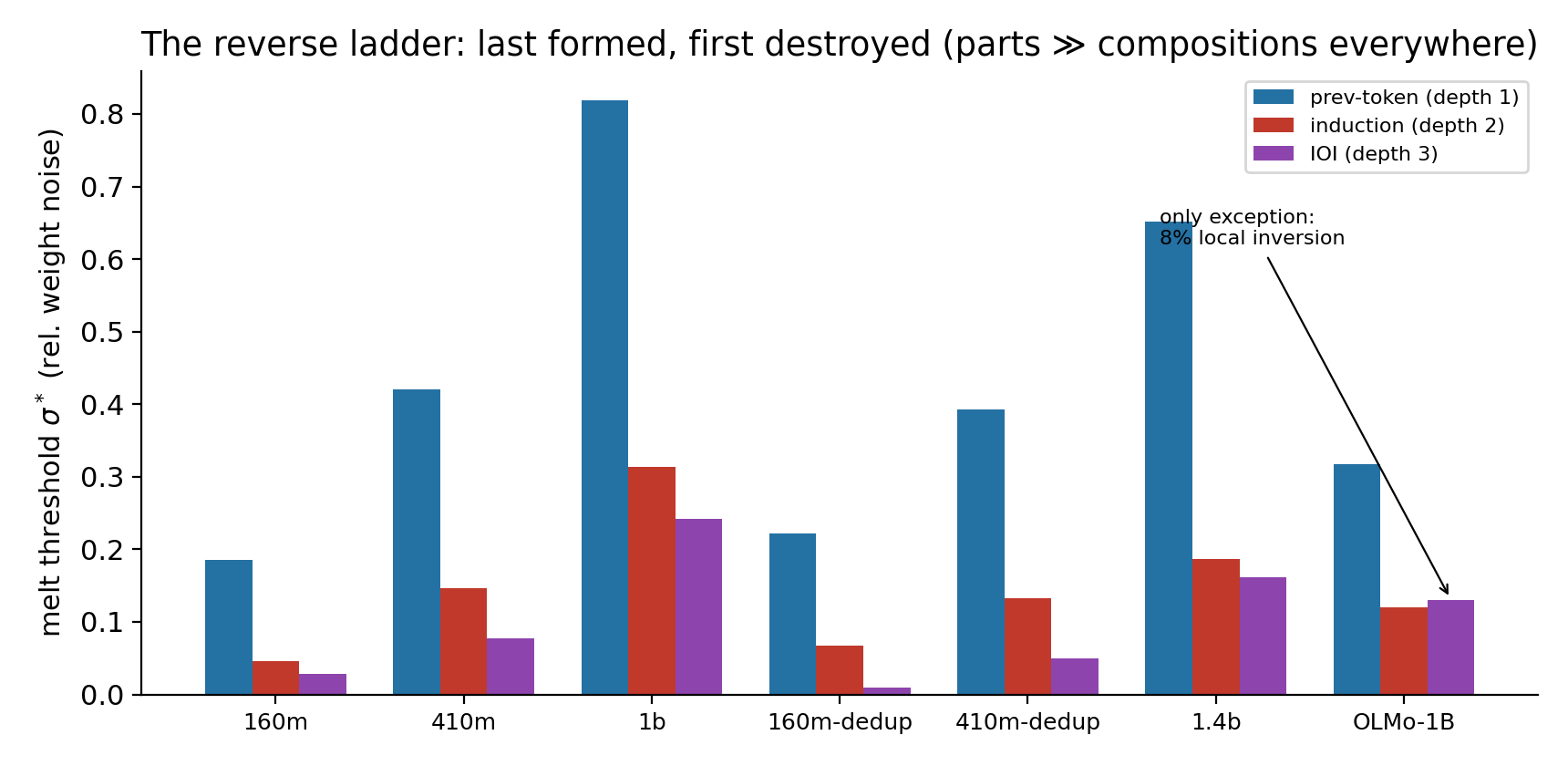}
\caption{\textbf{The reverse ladder}: melt thresholds order parts $\gg$ compositions at every scale and both families; last formed, first destroyed --- Prop.~\ref{prop:melt}(b) measured.}
\label{fig:lifo}
\end{figure}

\subsection{Glass, quench, and the release protocol}

Between the nose and the melt line lies the third control primitive: \emph{pinning}. Holding a bath at window temperature pins every sub-threshold structure at a steady-state order parameter $\phi_{ss}(T)$ --- measured as a clean decreasing spectrum ($0.373/0.303/0.204/0.164/0.106$ at $T = 0.003..0.02$, seed spread $\pm0.02$) with $\sim$100$\times$ retardation of ordering, the driven-steady-state picture of irradiated alloys \citep{martin1996driven}. Release is \emph{not} merely removing the bath: on real text, pure-temperature release fails (0/16) because quenching freezes the disorder in --- the burned state has weights inflated 5--24$\times$ and attention entropy collapsed to saturation. The working protocol is quench \emph{plus mechanism surgery}: rescaling attention-weight RMS to healthy values at release (direction preserved) recovers 6/6, and attention-only surgery re-ignites in $\sim$500 steps --- faster than a virgin substrate, the held heat converted into a head start. Damage bookkeeping is a dose law: below the knockdown line, zero accumulated damage; above it, damage $\propto$ knockdown rate $\times$ duration ($R^2$ 0.952) --- discrete Poisson-avalanche knockdowns, not diffusive wear, consistent with the jump-type destruction term.

The control palette, then: \emph{seeds and sites choose which} (a transplanted sub-critical seed re-orders barriers --- the only lawful way to invert the temperature-forced order, 8/8); \emph{drive chooses whether and how fast} (floor and throttle); \emph{temperature chooses when} (pin at $\phi_{ss}$, release by quench-plus-surgery); \emph{the melt line chooses what survives}. Each primitive is a term of Eq.~\eqref{eq:master}.

\section{The engine, the synthesis, and the boundary}
\label{sec:engine}
\label{sec:synthesis}

Two duties remain to the spine: measuring the equation's one non-obvious \emph{input} ($K$, the part count), and testing its \emph{structure} (the factorization) where all axes move at once.

\textbf{The engine.} Given only a behavioral probe, single-head ablation at a just-post-ignition checkpoint returns the circuit's constituent parts; tracking each constituent's attention-pattern crystallinity across checkpoints reveals which arrived first --- that part is the precursor, and its count is $K$ (Fig.~\ref{fig:engine}). Blind runs recover the known circuits in 10/11 capability--scale pairs. The count is physically load-bearing, not descriptive: induction ($K = 2$, $C = 1.24$) and IOI away from the wall ($K = 4$, $C = 2.1$) share $\beta \approx 0.23$, whereas the textbook depth assignment ($K = 3$ for IOI) is inconsistent with the clock --- the machine count is the one the physics accepts. Nor is the count circular: $K$ is measured by ablation at a \emph{single} checkpoint, with no access to formation times, and only then inserted into the clock --- an independently measured input that could have failed to unify the constants (the textbook $K{=}3$ \emph{does} fail), not a parameter tuned to make them agree.

\textbf{The count's domain of validity, measured.} We delimited it directly, by sweeping single-head ablation of the copy capability across four scales and two checkpoints (behavioural probe; chance $\approx 0$). Two facts emerge, and the second is a real limitation. \emph{First}, the count is right where the theory uses it: at a just-post-ignition checkpoint in the smallest model the individually-necessary set is exactly $\{$previous-token head, induction head$\}$ --- $K = 2$, both known parts recovered by a criterion that had no access to their identity --- and at 160M it is again of size 2. \emph{Second}, the criterion degrades with redundancy: by the final checkpoint at 410M and 1B \emph{no single head} clears the necessity threshold ($K \to 0$; ablating any one head leaves accuracy at $0.95$ against a $0.95$ baseline), because the capability has delocalized into interchangeable copies. A redundancy-robust variant --- the smallest \emph{group} whose joint ablation destroys the capability --- grows from 1--2 in the small/early models to 3 in the large/late ones, quantifying the redundancy, but it measures the \emph{destruction} side (a minimum over links, Prop.~\ref{prop:melt}) and therefore does not substitute for the formation-side count. The honest statement is thus sharper than before: ablation-based part counting is valid \emph{before delocalization} --- which is exactly the engine's stated operating rule --- and beyond it the formation-side $K$ must be carried by formation-time evidence (the clock's cross-circuit consistency), not by ablation. That the causal precursor (the previous-token head) remains in the necessary set at every scale and checkpoint where the criterion is defined is what the clock actually relies on, and it is measured to hold throughout.

Relatedly, the same sweep measures the AND-gate structure of real circuits directly: removing \emph{one} identified part collapses the capability to a median $5\%$ of baseline (range 2--42\%), whereas removing a randomly chosen non-part head leaves it at 99--102\% of baseline. Graceful, proportional degradation --- the signature of a partial-credit architecture --- is absent; the collapse is what a conjunction predicts, and it is not an artifact of ablation \emph{per se}, since the control ablations do nothing. The engine's one systematic failure is itself confirmatory: at late checkpoints, formed circuits delocalize into redundant copies and single-head ablation loses signal --- fixing the operating rule (dissect shortly after ignition) and independently evidencing the redundancy physics.

\textbf{The AND-gate test at scale, on the engine's own operating rule.} Applying that rule turns the test from anecdote into statistics. Sweeping three scales (70M/160M/410M) $\times$ five checkpoints (1k--16k steps, i.e.\ near ignition rather than at the end of training) $\times$ seven capabilities gives 120 cells, of which 32 are \emph{discriminating} --- capability present, and at least two individually necessary parts, so that conjunction and partial credit predict different things ($K{=}1$ cells are excluded a priori: with one part both models predict collapse). Removing one part leaves a median $16.7\%$ of the capability; removing a random non-part head leaves $100\%$. Against each cell's own partial-credit prediction ($1-1/K$, i.e.\ 50--83\% retained), \textbf{32 of 32 cells fall below}, sign test $p = 2.3\times10^{-10}$; the stricter preregistered criterion ($\le 25\%$ retained \emph{and} controls $\ge 80\%$) passes 22 of 32, the remainder landing at 25--48\% --- still below partial credit, not above it. Part counts span $K = 2$--$6$ across IOI, subject--verb agreement, greater-than, in-context learning, word-pair recall, counting, and copying. Six of the 32 were recovered by a preregistered secondary read-out: where baseline accuracy saturates at 100\% the instrument has no dynamic range, so we score the unthresholded margin (accuracy is the sign of the margin) --- a resolution fix of the same kind as measuring probes in fp32, recorded per cell so the two read-outs are never mixed.

\begin{figure}[htbp]\centering
\includegraphics[width=0.62\linewidth]{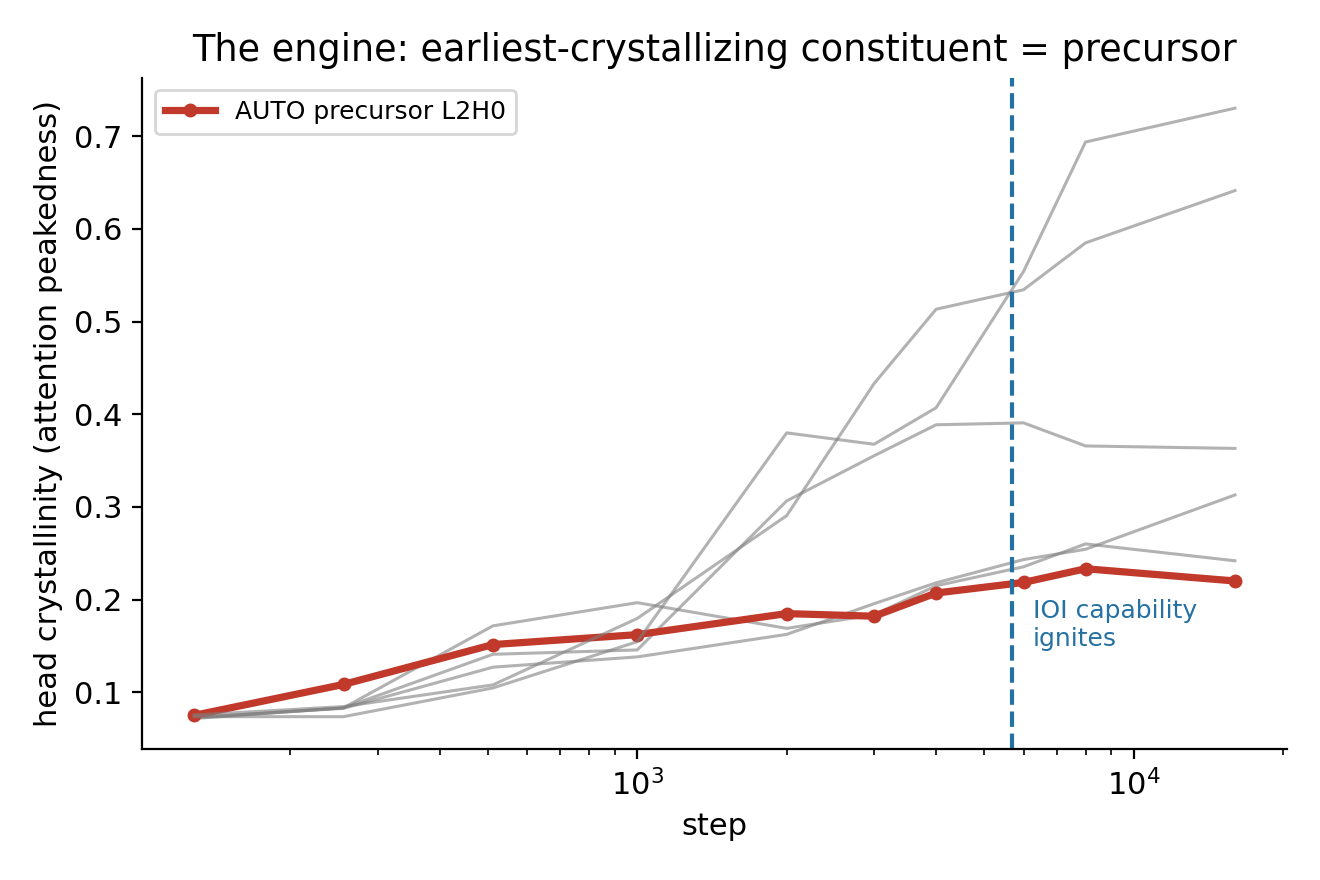}
\caption{\textbf{The engine}: ablation finds the parts; the earliest-crystallizing constituent is the precursor.}
\label{fig:engine}
\end{figure}

\textbf{The synthesis, and the boundary as a signed prediction.} Full frozen constants (continued-training regime): sites decay slope $0.61\pm0.09$; batch exponent $-0.329\pm0.046$ ($R^2$ 0.96); drive exponent $\gamma = 0.62\pm0.01$ ($R^2$ 0.999); barrier $\beta = 0.23\pm0.05$ (Pythia, off-wall); melt thresholds per Fig.~\ref{fig:lifo}. Freezing all single-axis terms and predicting three never-run combinations in which all three axes deviate at once: 8.4\% and 1.3\% error in the bulk; 112\% at the near-floor point, the error growing monotonically with floor proximity ($\mathrm{error} \propto (c - c_0)^{-3.2\pm0.9}$, $R^2$ 0.92 over three points, stated as such). By Corollary~\ref{cor:dependence} this is not an unexplained failure but the predicted signature of part dependence where capacity is shared: the factorized law holds where independence holds and couples at the boundary --- the standard structure of nucleation kinetics away from versus near a spinodal --- and it is a standing falsifiable commitment: any new combination's factorization error should be set by its floor distance.

\section{Why one law survives a complex model: the rate-determining step}
\label{sec:rds}

A fair objection runs: a language model contains many circuits, coupled far more richly than a bag of $K$ parts --- dataflow chains across layers, fan-in, cooperative and competitive interactions --- so why should a single exponential law with a single $K$ describe any real capability? The answer is not ours to claim. It is the \emph{rate-determining step} principle of chemical kinetics \citep{laidler1987chemical}: a multi-step mechanism, however elaborate, exhibits the rate law of its slowest step, and the complexity of the remaining steps collapses at the bottleneck. One Arrhenius form underwrites thousands of distinct mechanisms for exactly this reason. What is ours is the demonstration, measurement by measurement, that capability formation \emph{has} such a bottleneck, that the bottleneck is the subset-signal-free joint alignment, and that its dimensionality is observable:

\emph{(i) The cheap steps are cheap and happen early.} Circuit stages with independent utility --- parts that carry their own gradient signal --- form first and fast: the mechanism precedes the behavior in 15/15 model--task pairs (\S\ref{sec:emergence}), and an exact-depth from-scratch ladder ignites strictly in order, depths $1\to4$, in 6/6 seeds. \emph{(ii) The expensive step is the joint alignment --- and this dichotomy is itself a theorem.} Where no strict subset carries signal, a single added part multiplies formation time by $>$11$\times$ (the parity family), against 5.8$\times$ when subset signal exists (the summed family) --- the bottleneck's cost dwarfs the stages'. The staircase/leap theorems give this measured dichotomy its mathematical backbone: subset-signal-at-every-stage targets are SGD-learnable in $O(d)$ samples (an iff in the mean-field regime), while an $\ell$-wise joint step costs $\tilde\Theta(d^{\max(\ell-1,1)})$ where proven \citep{abbe2022merged,abbe2023leap} --- cheap stages and an expensive bottleneck are not our metaphor but the proven structure of SGD learning on this function class. \emph{(iii) Everything upstream of the bottleneck is kinetically invisible.} This is the direct rate-determining-step measurement: in the missing-parts grid (EK-5), waiting time depends only on the size of the remaining joint step, not on how much circuit already stands --- rows flat to 1.1--1.4 across total $K$. \emph{(iv) The wait is governed by one barrier, not a chain of comparable ones.} The waiting-time shape is exponential ($\kappa \to 1.2$, EK-1); a sequence of $k$ comparable activated steps would instead show an Erlang shape $\kappa \approx k$. \emph{(v) The bottleneck's dimensionality is measurable and load-bearing.} The engine's causally-counted $K$ is what unifies the clock across different circuits ($K{=}2$ and $K{=}4$ sharing one $\beta$) while the textbook count fails (\S\ref{sec:engine}).

Nor do circuit--circuit interactions reopen unbounded complexity: each measured interaction enters as a \emph{correction with a measured form}, not a new law. Cross-circuit catalysis makes the rate linear in neighbor density (secondary nucleation; $R^2$ 0.88 on real models in prior work \citep{lei2026circuits}); competition for shared capacity renormalizes constants rather than the form (in a two-skill furnace the loser's formation time is renormalized $\sim$20$\times$ while each skill alone obeys its own law); and the near-wall coupling carries its own scaling, $\mathrm{error} \propto (c - c_0)^{-3.2}$ (\S\ref{sec:synthesis}) --- the theory meters its own domain of validity.

The epistemic position this leaves us in is, again, not novel; it is the standard position of physical metallurgy \citep{porter1992phase}. Nobody derives a steel's TTT diagram from quantum mechanics --- steel is far too complex --- yet the framework fixes the diagram's \emph{structure}, a handful of measured points fix the alloy's constants, and the result runs industry. Here the structure (floor, clock, shelf life, melt line) transfers across families, the per-family effective $\beta$ is the material constant, and one cheap anchor --- a single attention readout at 0.36\% of training --- calibrates a family, after which frozen constants predict held-out models to median 5\% (\S\ref{sec:emergence}). The framework would break, and is built to break visibly, if the rate-limiting step showed non-exponential kinetics, if the precursor stopped dating the composition, or if the causal part-count stopped unifying the constants; none of these has happened in any system tested.

\section{Negative results and reproducibility}
\label{sec:negative}

\textbf{A design lesson worth more than the result it cost us.} Synthetic sequence tasks harbour \emph{algebraic shortcuts}: alternative solutions of much lower part count that agree with the intended target on the training distribution. We have now hit this family three times --- a closed-loop collapse in which a fully-driven conjunction degenerates into single-token copying; a look-up-table task solved by periodic orbits; and the summed design's two-term recurrence documented above, which silently replaced the intended $K$-part solution from $K{=}3$ upward and cost us a per-part barrier number we have withdrawn. The general prescription we now follow, and recommend: for any synthetic target, (i) search for algebraic identities that relate the target at nearby positions, since these generate low-part-count alternative solutions; (ii) build a \emph{discriminating probe} --- evaluation positions where the shortcut and the intended solution disagree --- and report both scores, never the aggregate alone; (iii) prefer designs that are provably shortcut-free (our query-parity design isolates each query behind a gap, so consecutive targets are statistically independent and no cross-position identity exists). A task whose intended solution is never verified against its cheapest alternative is not a measurement of that solution.

The complete failure log is part of the evidence. IOI is un-ignitable in every continued-training arm (its barrier exceeds the regime's power, consistent with its native $\tstar$); greater-than is unwithholdable (plain text already exceeds its floor --- capabilities split into withholdable and unwithholdable by food scarcity); a global entropy regularizer does not rescue aged substrates (it also suppresses the peaked circuit it should enable); newborn circuits are narrow-band (food at repetition period 128 does not serve a period-64 probe --- a trap hit twice, fixed, and itself evidence that young circuits are distance-specific before delocalization); four attempts at a precursor dose--response were all confounded by precursor redundancy (the causal case rests on the switch and the resets instead); pure-temperature release fails without surgery (\S\ref{sec:control}); three execution artifacts were caught by internal audits (a mirror resolving every revision to one commit --- detectable because step-0 weights scored as trained; \texttt{.bin}-only checkpoints refused under a torch CVE policy; an ignition-detector edge case), each with detection criteria for reproducers; and one publicly failed frozen prediction (the occupation tax was frozen as saturating and measured as divergent) is retained because a fitted story would not contain it. Every number is recomputable from released per-experiment JSON files and scripts; the EK premise-test suite (\S\ref{sec:premise-tests}) was likewise designed and its budgets frozen before running.

\section{Honest scope}
\label{sec:scope}

$\le$1.4B parameters; two families (the only ones with dense public checkpoints); causal constants belong to a continued-training regime whose token flux is far below native pretraining (the from-scratch reproduction removes the pretrained-substrate confound but not the flux gap); capabilities are mechanism-level plus one task-level demonstration --- no benchmark-suite prediction yet; switch/shelf/reset arms at $n = 5$, several others at $n = 2$--3; single hardware/framework stack; fp32 is required for \emph{probe evaluation only} (softmax scores collapse under bf16 underflow) --- models trained in low precision are probed in fp32, so the constraint is on the measurement instrument, not on training practice; the diffusion-approximation regime of Step 1 is itself a scope condition (production optimizer stacks deform the SDE); the engine's count for induction is less crisp than for IOI; direct-$K$ measurements span $K = 2$--$4$, and the wide-$K$ ladder (Remark~\ref{rem:prefactor} ff.) is the stated next experiment; $\beta$ is single-valued only away from the capacity wall; control results are toy-plus-real-Pythia, not production-scale; the alignment-fragility reading (safety training as the top rung --- last in, first out \citep{qi2023finetuning}) is a corollary of Prop.~\ref{prop:melt}(b) plus literature, not a direct experiment. The mathematical status is exactly as marked, in three grades. Proved: the no-partial-credit structure of the apparatus (Lemma~\ref{lem:nopartial}); the exponential-in-$K$ cost for the task class (query-complexity and leap theorems, \S\ref{sec:theory}A, with their stated regimes: SQ/CSQ algorithm classes, single-pass training, large ambient dimension --- and no such theorem yet exists for transformers themselves); and the occupation theorem for the backward reading (Theorem~\ref{thm:occupation}, App.~\ref{app:occupation}) --- deadline, data-irrecoverability, reset-repair, and loss-silence, its confinement side free of any mixing assumption, at the price of the model's stated idealizations (positional attention, oracle gating, the C0 feasibility condition), each idealization flagged in the appendix with its empirical counterpart. Conditional: the nucleation-form propositions, exact given Assumptions~\ref{as:mixing}--\ref{as:indep}; Assumption~\ref{as:mixing}(b) --- mixing for transformer SGD on formation timescales --- now burdens \emph{only the physical reading} (temperature, supersaturation, melting), is beyond present mathematics for anyone, and proving it would upgrade that reading to a theorem: an invitation, not a concession. Measured: everything else, preregistered.

\section{Related work: the equilibrium picture, and what a rate adds}
\label{sec:related}

\textbf{Equilibrium theories of training structure.} Singular learning theory and developmental interpretability characterize training stages as Bayesian free-energy transitions and detect them post hoc via the local learning coefficient \citep{watanabe2009algebraic,hoogland2024developmental,chen2023dynamical}; grokking has an equilibrium first-order model \citep{rubin2024grokking} and a Kramers-escape model in linear networks \citep{ersoy2026noise}. These supply the statics. None yields a formation-\emph{rate} law, a precursor clock, an aging horizon, or a control schedule --- a free energy is not a rate. \textbf{Predicting emergence.} \citet{aoyama2025predicting} predict induction-head arrival from data statistics with a throughput equation; the prediction is real but mechanism-free. Our clock is derived from a barrier law, reads an \emph{internal} precursor, and its $K$-dependence is forced, not fit; and where \citet{lavie2026phase} argue softmax-attention emergence is intrinsically unpredictable and \citet{baherwani2026random} that timing is initialization-random, the prospective clock results are a direct counterexample at the capability class we cover. \textbf{Emergence phenomenology.} \citet{wei2022emergent,schaeffer2023mirage,ganguli2022predictability} debate description; \citet{olsson2022context,elhage2021mathematical} identify the induction circuit and its phase change; \citet{reddy2024mechanistic,singh2023transient,hoogland2024developmental} model and detect its development. We add the quantitative ratio law, frozen-constant prospective validation, cross-family transfer, causal steering, and the plasticity connection. \textbf{Data mixing.} \citet{gu2026mixing} document a mixing-ratio threshold and attribute it to capacity allocation; our floor carries the nucleation signatures (intensivity, near-floor divergence, scale trend) in a mechanistic order parameter (Prop.~\ref{prop:floor}); \citet{xie2023doremi,liu2024regmix,ye2024mixing} optimize loss, not which capabilities form. \textbf{Loss of plasticity.} The phenomenon and remedies are established \citep{dohare2024loss,ash2020warm,sokar2023dormant}; the causes literature is explicitly partial \citep{lyle2024disentangling}; \citet{lyle2026grokking} conjecture the emergence--plasticity link via effective learning rate. We supply the quantitative form: one rate equation with law, horizon, located mechanism, repair, and the site/drive discriminator (Prop.~\ref{prop:sites}). Resets as a \emph{remedy} are standard in that literature \citep{dohare2024loss,sokar2023dormant}; what is new here is not the operation but the causal localization of \emph{which} weights (attention --- with MLP and embedding controls discriminated causally), \emph{why} (commitment, $r = 0.93$), and \emph{when} (the horizon and its scale law). \textbf{Non-equilibrium SGD.} Broken detailed balance and irreversibility are established \citep{chaudhari2018stochastic,ziyin2026irreversibility}; Yaida's fluctuation--dissipation relations are a measurement tool at stationarity \citep{yaida2019fluctuation}; we add the FDT \emph{qualification} test, its six-knob failed verdict, and the injected-bath prescription. \textbf{Control baselines.} Injected-noise training is itself classical --- simulated annealing, gradient noise \citep{neelakantan2015adding}, SGLD --- but is framed there as optimization or exploration; none applies a fluctuation--dissipation qualification, none selects \emph{which} circuit forms, and none demonstrates pinning or selective dissolution --- the operations, not the noise, are the contribution. Grokfast \citep{lee2024grokfast} accelerates grokking by gradient filtering without selecting circuits; dual-process steering \citep{singh2024dual} switches mechanisms by discrete forgetting; our schedules are continuous, bidirectional, and derived from the rate equation. \textbf{Fragility at coarse grain.} Layer- and module-level sensitivity to perturbation and quantization is known; the circuit-grain melt points, their cross-scale constancy, and the build/break duality (Prop.~\ref{prop:melt}) are new. \textbf{Metaphor.} A crystallization framing of training has been proposed qualitatively \citep{hallgren2025crystals}; this program makes it quantitative and falsifiable. \textbf{Our own prior work} develops single-circuit nucleation with measured Langevin structure \citep{lei2026circuits}; this paper is the generalization to the full life cycle and the control layer.

\section{Conclusion}

Training has had a phase diagram; this paper gives it a TTT diagram. One rate equation --- sites $\times$ attempts $\times$ supersaturation $\times$ barrier, minus destruction --- derived under stated assumptions and measured factor by factor, carries a capability from its first invisible precursor through ignition, aging, and, if one chooses, controlled dissolution. Emergence is the equation read forward: a floor, a throttle, and a clock that dates arrival to median 5\% before the capability exists. Loss of plasticity is the same equation read backward: a tax, a horizon, a located mechanism, a repair --- and, in a controlled model, a theorem --- the quantitative law a decade of remedies has been missing. Control is the equation completed: install the temperature SGD lacks, and circuits can be accelerated at the nose, pinned in the glass, released on command, and dissolved selectively at their melt points --- scheduled, in short, like the heat treatment of an alloy. The mathematics is exactly as strong as its premises, and the premises are the experiments: joint alignment is a product law (measured), the barrier is additive (measured, with its own deviation theory), the waiting is memoryless (measured), and the one unproved premise --- mixing for transformer SGD --- is stated as the field's next theorem to want.

\FloatBarrier
\appendix
\section{The occupation theorem: model, assumptions, and full proofs}
\label{app:occupation}

This appendix proves Theorem~\ref{thm:occupation}. Every step is graded: \emph{exact} (arithmetic or algebraic identity), \emph{in-model theorem} (proved under the stated assumptions), \emph{cited theorem} (an external result with its conditions verified), or \emph{explicit assumption}. The headline structural fact bears repeating: clauses (ii)--(iv) --- the consequences of occupation --- use only a drift inequality and bounded increments (Hajek \citep{hajek1982hitting}); no ergodicity or mixing enters. Only clause (i), the deadline's \emph{arrival}, uses the recruitment assumption (G8), whose empirical counterpart is the measured occupation of \S\ref{sec:plasticity}.

\subsection{Model}
\label{app:occ-model}

\textbf{Data.} Each sequence position belongs independently to one of: a \emph{background skill} $i \in \{1..M\}$ (share $\pi_i$): a cue token $g_i$ appears and the correct next token is $x_{t-o_i}$, a copy from fixed offset $o_i$; the \emph{target} (share $c$): a query token $Q$ appears and the correct next token is $Y = \big(\sum_{j \in R} x_{t-j}\big) \bmod V$ with $R \subset \{1..W\}$, $|R| = K \ge 2$, window contents i.i.d.\ uniform on $\mathbb{Z}_V$ (the apparatus of Lemma~\ref{lem:nopartial}); or \emph{neutral substrate} (remaining share): i.i.d.\ uniform tokens.

\textbf{Architecture.} One layer of $H$ heads. Head $h$ carries positional scores $s^{(h)} \in \mathbb{R}^W$ with attention $a^{(h)} = \mathrm{softmax}(s^{(h)})$, and outputs $u_h = \sum_j a^{(h)}_j e(x_{t-j})$ with $e$ the one-hot value embedding; thus $u_h \in \mathrm{conv}\{e_1,\dots,e_V\}$ and the joint domain $\mathcal{U} := \{(u_1..u_H)\}$ is convex and compact. The readout is \emph{gated by the current token}: $z = F(u;\, g)$, with a background channel (value matrix $B$, plus a bounded higher-order part $F_2$) on cue tokens, a target channel $F_Q$ on $Q$, and the substrate channel \emph{closed} ($z$ constant, independent of $u$).

\textbf{Training.} SGD with learning rate $\eta$, gradients clipped in norm at $G$, minibatches sampled unbiasedly from the stream.

\textbf{Definitions.} Head $h$ \emph{carries} skill $i$ if $a^{(h)}_{o_i} \ge 1 - \delta_0$, the open-channel value margin is positive (G4), and $h$ is the unique carrier (zero-ablating $h$ collapses the margin --- the operational, measured definition). With $\mu_i$ as in C0 below, a carrier is \emph{firmly pinned} if $\mu_i > 0$. $N_{\rm avail}(t) :=$ number of heads not firmly pinned; $t_{\rm dead} := \inf\{t : N_{\rm avail}(t) < K\}$. Serving a target slot $r$ requires $a^{(h)}_r \ge \gamma_{\rm need} > \delta_0$; \emph{formation} $=$ $K$ heads serving the $K$ slots of $R$. For a carrier, the score-gap coordinate is $\xi_j := s^{(h)}_j - s^{(h)}_{o_i}$ (so $a_j / a_{o_i} = e^{\xi_j}$), and $\ell_k := \partial\, \mathbb{E}[L] / \partial a_k$, $\langle \ell \rangle_a := \sum_k a_k \ell_k$.

\subsection{Assumptions and conditions}
\label{app:occ-assumptions}

\begin{assumption}[The gated model]
\label{as:gated}
\textbf{(G1) Readout locality:} logits depend on window contents only through $(u_1..u_H)$ --- no trainable bypass from raw tokens. \textbf{(G2) Positional attention:} one global pattern per head. \textbf{(G3) Disjointness:} $\{o_i\} \cap R = \emptyset$. \textbf{(G4) Carrying margin:} value-channel margin $\gamma := \min_x [B_{xx} - \max_{y \ne x} B_{yx}] \ge \gamma_0 > 0$ for carried skills. \textbf{(G5) Regularity:} row-mean balance $|\bar B_y - \bar B_{y'}| \le \gamma/4$ with $\bar B_y := \frac1V \sum_x B_{yx}$; each channel is $C^{K+1}(\mathcal{U})$ with mixed partials $\le C_F$; $\beta_2 := \sup_{\mathcal U} \|\nabla_u F_2\|$; $\beta_3 := \sup_{\rm basin} |\partial_\xi\text{-differential of } \nabla\Phi_Q|$; $\Lambda := \sup_{\rm basin} \max_k |\ell_k - \ell_{o_i}|$ (all finite: $\mathcal{U}$ compact). \textbf{(G6) Oracle gating:} the substrate channel is closed. \emph{This is the model's most idealized assumption: it presupposes perfect task identity. Removing it breaks the theorem. Its internal necessity: with the substrate channel open, the uniformizing force (Lemma~\ref{lem:decomp}) acts at the largest data share and penalizes large attention masses hardest --- in the score-gap coordinate an escape-\emph{assisting} force with no small factor --- and background skills themselves become unlearnable at equilibrium (junk penalty exceeds skill reward). Gating is what a substrate trained to equilibrium looks like; real transformers implement it by content routing. Empirical counterpart: trained models output near-uniform on substrate positions, and control ablations retain 99\%.} \textbf{(G7) Leakage bound:} on skill-$i'$ positions the readout's coupling to non-carrier heads yields a score-gap differential pull $\le \chi$. \textbf{(G8) Recruitment (ergodic-flavored; used only for clause (i)):} while $\ge K$ heads are available, each uncarried high-traffic skill has per-step recruitment hazard $\ge \rho_{\min} > 0$.
\end{assumption}

\textbf{Conditions.} With $\lambda'$ and $J$ defined in Lemmas~\ref{lem:restoring} and \ref{lem:weakcoupling}: \textbf{C0 (feasibility):} at least $H - K + 1$ skills have $\mu_i := \pi_i \lambda' - c(J + \beta_3) - \chi - \beta_2 > 0$. (Uniform special case $\pi_i = (1-c)/M$: C0 reads ``$M$ not so large as to dilute the pull-back.'' The constants inside $J$ are conservative; effective serving-head counts would tighten them.) \textbf{C2 (small step):} $2\eta G \le 1$. \textbf{C3 (timescale separation):} with tolerance $\delta_{\rm tol}$ and $T_{\rm rec} := (H{-}K{+}1)/(\rho_{\min} p_0)$: $T_{\rm rec} \le (\delta_{\rm tol}/2H)\, e^{\theta(2 - 2\eta G)}$ ($\theta$ from Lemma~\ref{lem:confinement}; automatic at small $\eta$ since the right side is exponentially large). \textbf{C4 (small basin):} $\delta_0 \Lambda \le \mu_i$.

\subsection{Pre-formation weak coupling}

\begin{lemma}[Conditional uniformity; exact]
\label{lem:condunif}
For any $S \subsetneq R$ and any window coordinates outside $R$, conditioning on them leaves $Y$ uniform on $\mathbb{Z}_V$.
\end{lemma}
\begin{proof}
Pick $j^* \in R \setminus S$: $Y = (\text{const} + X_{j^*} + \cdots) \bmod V$ with $X_{j^*}$ uniform and independent of the conditioning; uniformity is translation-invariant. Hence any predictor reading a strict subset has cross-entropy $\ge \ln V$, with equality only for the constant uniform predictor.
\end{proof}

\begin{lemma}[Fourier support; exact]
\label{lem:fourier}
With $\omega = e^{2\pi i/V}$ and $\chi_m^R(x) := \prod_{j \in R} \omega^{m x_j}$: $\mathbf{1}_{Y=y} = \frac1V \sum_m \omega^{-my} \chi_m^R(x)$; for $m \ne 0$ the character's support is exactly $R$; and $\mathrm{Cov}(\mathbf{1}_{Y=y}, \varphi) \ne 0$ iff $\varphi$ has a Fourier component with support $\supseteq R$.
\end{lemma}
\begin{proof}
Discrete inversion on $\mathbb{Z}_V$; character orthogonality under the i.i.d.\ uniform measure.
\end{proof}

\begin{lemma}[Iterated-oscillation bound; in-model theorem]
\label{lem:weakcoupling}
Let $A_r := \sum_h a^{(h)}_r$. Under (G1) and (G5), every Fourier coefficient of $z$ or of $\partial z/\partial s^{(h)}$ with support $\supseteq R$ is bounded by $C_F\, 4^K \prod_{r \in R} A_r$. Consequently, pre-formation (some slot $r^*$ has $A_{r^*} \le H\bar\alpha$), the label-dependent (signal) part of the target gradient on any head's scores is bounded by $c J$ with $J := V C_F 4^K H^K \bar\alpha$.
\end{lemma}
\begin{proof}
\emph{(Coefficient bound.)} For $m \ne 0$, $\mathbb{E}_{X_r}[\omega^{-mX_r}] = 0$, so replacing $g$ by $g - \mathbb{E}_{X_r}[g \mid \text{rest}]$ leaves $\langle g, \bar\chi_m^R\rangle$ unchanged while its sup-norm is bounded by the direction-$r$ oscillation; iterating over the $K$ coordinates of $R$ bounds the coefficient by the mixed $K$-fold oscillation. Changing $X_r$ displaces $u$ by $v_r := \sum_h a^{(h)}_r (e(x_r') - e(x_r))$, $\|v_r\|_1 \le 2A_r$, and the segment $[u, u + v_r]$ stays in the convex $\mathcal{U}$, so the mixed difference has the iterated line-integral representation $\int_0^1\!\!\cdots\!\int_0^1 D^K F[v_{r_1},..,v_{r_K}]\, d\lambda$, giving $|{\cdot}| \le C_F \prod_r 2A_r$ --- the mean-value step is legitimate precisely because $\mathcal{U}$ is convex. For $\partial z / \partial s^{(h)}_j = \langle \nabla_{u_h} F,\, a_j^{(h)}(e(x_j) - u_h)\rangle$, both factors depend on contents; the discrete Leibniz rule $\Delta_r(fg) = (\Delta_r f) g + (S_r f)(\Delta_r g)$ iterated over $K$ coordinates yields $2^K$ terms, each bounded by $C_F \prod_r 2A_r$, for a total $C_F 4^K \prod_r A_r$. \emph{(Signal bound.)} By Lemma~\ref{lem:fourier} the label-dependent part flows only through support-$\supseteq R$ components; summing $V$ frequencies and bounding $\prod_r A_r \le H^{K-1} \cdot A_{r^*} \le H^K \bar\alpha$ gives the claim. A linear readout ($K$-th mixed derivative zero) has \emph{exactly zero} signal drift.
\end{proof}

\begin{lemma}[Exact gradient decomposition; in-model theorem]
\label{lem:decomp}
On target positions, for any parameter block,
$\mathbb{E}[\nabla L_T] = \nabla \Phi + S$, where $\Phi := \mathbb{E}\big[-\tfrac1V \sum_y \ln p_y\big]$ (the \emph{uniformizing potential}, label-free, globally minimized exactly on the chance manifold $p \equiv \mathrm{uniform}$ by Gibbs' inequality) and $S := -\sum_y \mathrm{Cov}\big(P(Y{=}y \mid X),\, \nabla z_y\big)$ with $|S| \le cJ$ pre-formation.
\end{lemma}
\begin{proof}
$\mathbb{E}[\mathbf{1}_{Y=y} \nabla z_y] = \mathbb{E}[P(Y{=}y|X) \nabla z_y]$ (tower property; the model's forward pass never sees $Y$). Split $P(Y{=}y|X) = \frac1V + (\text{full-support part})$ by Lemmas~\ref{lem:condunif}--\ref{lem:fourier}; the $\frac1V$ part joins $\sum_y \mathbb{E}[p_y \nabla z_y]$ to form $\sum_y \mathbb{E}[(p_y - \frac1V)\nabla z_y] = \nabla\Phi$ (softmax identity $\partial_{z_k}(-\frac1V\sum_y \ln p_y) = p_k - \frac1V$); the remainder is $S$, bounded by Lemma~\ref{lem:weakcoupling}. \emph{Level separation:} this lemma establishes the decomposition and the force's form only; \emph{where} the uniformizing force acts is governed by (G6)/(G7) and is book-kept in Lemma~\ref{lem:drift}.
\end{proof}

\subsection{Exclusivity and restoring drift}

\begin{lemma}[Exclusivity; exact]
\label{lem:exclusivity}
Under (G2), $a^{(h)} \in \Delta_W$ is a single global vector: carrying ($a_{o_i} \ge 1-\delta_0$) forces total target-slot mass $\le \delta_0$; serving a slot ($\ge \gamma_{\rm need} > \delta_0$) requires first shedding mass from $o_i$. \emph{(Rank-1 remark, supporting evidence only: content-dependent scores of rank one, $s_j(\mathrm{cue}) = \varphi(\mathrm{cue})\kappa_j$, admit at most two sharp patterns per head --- $\arg\max \kappa$ and $\arg\min \kappa$ --- so the theorem survives with $H \to 2H$. General rank is not claimed; extrapolation is carried by the experiments.)}
\end{lemma}

\begin{lemma}[Restoring drift; in-model theorem]
\label{lem:restoring}
Let head $h$ carry skill $i$ (G4). On skill-$i$ positions, moving attention mass $\eps$ from $o_i$ to any other slot obeys
$dL_i/d\eps \ge (1 - \bar p_c)\,(3\gamma/4 - C_1\delta_0)$ with $C_1 := 16\|B\|_\infty^2$, under $4\delta_0\|B\|_\infty \le \tfrac14$. Since softmax confidence is strictly below 1, $(1-\bar p_c) \ge \kappa := \tfrac1V e^{-(2\|B\|_\infty + \|\zeta\|_\infty)} > 0$ (weight decay bounds $\|B\|$, making $\kappa$ a bona fide constant). Write $\lambda' := \kappa\,(3\gamma/4 - C_1 \delta_0)$.
\end{lemma}
\begin{proof}
With correct token $c^*$ and perturbation $dz_y = d\eps\,(B_{y,X'} - B_{y,c^*})$, the identity $\sum_y (p_y - \mathbf{1}_{y=c^*}) v_y = \sum_{y \ne c^*} p_y (v_y - v_{c^*})$ places the \emph{entire} first variation on error-probability weight. Three blocks: the diagonal block contributes $\ge \gamma \sum_{y\ne c^*} \mathbb{E}[p_y] = \gamma(1-\bar p_c)$ (G4); the row-mean block is bounded by $(\gamma/4)(1-\bar p_c)$ (G5 --- note it, too, carries the $(1-\bar p_c)$ weight, which is what forbids a sign flip at high confidence); the correlation error (the dependence of $p$ on $X'$ through the $a_{j'} \le \delta_0$ channel) is bounded along the logit-interpolation path by $C_1 \delta_0 (1 - \bar p_c)$. Summing gives the claim: the restoring force weakens with confidence but never reverses.
\end{proof}

\subsection{Score-difference dynamics and the force ledger}

\begin{lemma}[Exact score-gap dynamics]
\label{lem:scorediff}
(a) $\partial\,\mathbb{E}[L]/\partial s_j = a_j(\ell_j - \langle \ell\rangle_a)$ --- the softmax chain identity, with no approximation. (b) Score-gap updates are exactly linear in the stochastic gradient, so
$\mathbb{E}[\Delta \xi_j \mid \mathcal{F}] = -\eta\big[a_j(\ell_j - \langle\ell\rangle_a) - a_{o_i}(\ell_{o_i} - \langle\ell\rangle_a)\big]$, with \emph{no second-order remainder}.
\end{lemma}
\begin{proof}
$\partial a_k/\partial s_j = a_k(\delta_{kj} - a_j)$, so $\partial L/\partial s_j = \sum_k \ell_k a_k (\delta_{kj} - a_j) = a_j \ell_j - a_j \langle\ell\rangle_a$. $\Delta\xi_j = \Delta s_j - \Delta s_{o_i} = -\eta(\hat g_j - \hat g_{o_i})$; take conditional expectations.
\end{proof}

\begin{lemma}[Net drift; in-model theorem]
\label{lem:drift}
Under C0 and C4, inside the basin ($a_j \le \delta_0$ for all $j \in R$), for a firmly pinned head:
$\mathbb{E}[\Delta\xi_j \mid \mathcal{F}] \le -\eta\, a_{o_i}(1 - a_{o_i})\, \mu_i$ exactly; on the strip $a_j \in [\delta_0/e^2, \delta_0]$, $\mathbb{E}[\Delta\xi_j|\mathcal F] \le -\eta\,\mu_{\rm strip}$ with $\mu_{\rm strip} := (1-\delta_0)(\delta_0/e^2)\mu_i$, and $|\Delta\xi_j| \le 2\eta G$ (C2).
\end{lemma}
\begin{proof}
The pairwise gaps $\ell_{j'} - \ell_{o_i}$ are book-kept by position class, with \emph{no symmetry-cancellation argument anywhere}: skill-$i$ positions give $\ge \lambda'$ for every $j' \ne o_i$ at share $\pi_i$ (Lemma~\ref{lem:restoring}); other skills contribute $|{\cdot}| \le \chi$ (G7); substrate contributes zero (G6); target positions contribute $\le c(J + \beta_3)$ (Lemmas~\ref{lem:weakcoupling}--\ref{lem:decomp}); the higher-order readout $\le \beta_2$ (G5). Hence $\langle\ell\rangle_a - \ell_{o_i} \ge (1 - a_{o_i})\mu_i$, making the second term of Lemma~\ref{lem:scorediff}(b) at most $-\eta a_{o_i}(1-a_{o_i})\mu_i$; for the first term, $\ell_j - \langle\ell\rangle_a \ge \mu_i - (1-a_{o_i})\Lambda \ge \mu_i - \delta_0\Lambda \ge 0$ by C4, so it is favorable and can be discarded.
\end{proof}

\subsection{Confinement, occupation, and the theorem}

\begin{lemma}[Exponential confinement; cited theorem, self-contained proof]
\label{lem:confinement}
Under C0, C2, C4, with $\theta := \min\!\big(\mu_{\rm strip}/(8\eta G^2),\, 1/(2\eta G)\big)$:
$P\big(\exists\, t \le T: a_j(t) \ge \delta_0\big) \le T\, e^{-\theta(2 - 2\eta G)}$, so the expected escape time is $\ge e^{\theta(2-2\eta G)}/2$.
\end{lemma}
\begin{proof}
The $\min$ makes $\theta \cdot 2\eta G \le 1$ always; at the boundary $\theta\, |\Delta\xi| = 1$ the inequality $e^x \le 1 + x + x^2$ still holds ($e \le 3$), so on the strip $\mathbb{E}[e^{\theta \Delta\xi}|\mathcal F] \le 1 - \theta\eta\mu_{\rm strip} + \theta^2 (2\eta G)^2 < 1$ and $e^{\theta\xi}$ is a supermartingale. Escape must traverse the strip, of $\xi$-width $\ln(\delta_0 / (\delta_0/e^2)) = 2$; entries overshoot the bottom by at most $2\eta G$ per step; optional stopping bounds each crossing attempt by $e^{-\theta(2 - 2\eta G)}$, and at most one attempt begins per step, giving the union factor $T$. This is Hajek's drift theorem \citep{hajek1982hitting}; \emph{no ergodicity or mixing is used} --- only drift and bounded increments. The bound is conservative: it certifies difficulty of reaching $\delta_0 < \gamma_{\rm need}$, which reaching $\gamma_{\rm need}$ requires first.
\end{proof}

\begin{lemma}[Occupation is monotone; the deadline is finite]
\label{lem:occupationmono}
(a) Under C3, on the window $[0,\, 2T_{\rm rec}]$, the probability that \emph{any} firmly pinned head escapes is $\le H \cdot 2T_{\rm rec}\, e^{-\theta(2-2\eta G)} \le \delta_{\rm tol}$: $N_{\rm avail}$ is non-increasing with probability $\ge 1 - \delta_{\rm tol}$.
(b) Under (G8) and C0: recruitment events are \emph{not} independent (skills compete for heads); the argument uses only ``while $N_{\rm avail} \ge K$, at least one new firm pinning occurs per window with probability $\ge p_0$,'' so $T_{\rm rec}$ is stochastically dominated by a sum of $H-K+1$ geometric variables: $t_{\rm dead}$ is a.s.\ finite with $\mathbb{E}[t_{\rm dead}] \le T_{\rm rec} + C$, inside the window of (a) by C3.
(c) \emph{(Constant-hazard clause: stated only to record that we tested and rejected it.)} If in addition each free head faced a constant hazard $\ge \rho$, one would get $\mathbb{E}[t_{\rm dead}] \le \rho^{-1}\ln(H/(K{-}1)) + C$ and exponential site decay. \textbf{This extra clause is inconsistent with the model's own traffic-share structure} ($\mu_i \propto \pi_i$: once the high-traffic skills have taken their heads, the remaining skills pull weakly, so the hazard must fall), \textbf{and it is falsified by measurement} (App.~\ref{app:occ-measured}): the hazard declines by $\approx$15$\times$ across the sweep. We therefore drop it; only (a) and (b) are used.
\end{lemma}

\begin{proof}[Proof of Theorem~\ref{thm:occupation}]
(i) is Lemma~\ref{lem:occupationmono}. (ii): past $t_{\rm dead}$, formation requires some firmly pinned head to reach $\gamma_{\rm need} > \delta_0$ on some slot; Lemma~\ref{lem:confinement} plus a union over $H$ heads, $K$ slots, $T$ steps gives the bound. The share $c$ enters $\mu_i$ only through $-c(J + \beta_3)$: an $8\times$ increase in $c$ shifts $\mu_i$ linearly and the exponent barely moves; overturning C0 by data requires $c$ so large that $\pi_i \lambda' < c(J+\beta_3)$ --- destroying the background margin itself. (iii): re-initialized scores give near-uniform attention (mass $\approx 1/W < 1 - \delta_0$ on every offset), so the head carries nothing; $N_{\rm avail} \ge K$ is restored and the fresh-substrate rate applies; only attention \emph{scores} need resetting. (iv): pre-formation, $L_T = \ln V + \Delta_{\rm junk}$ with $\Delta_{\rm junk} := \Phi - \ln V \ge 0$ (Gibbs); the uniformizing force pushes it toward zero while leakage (G7) and the target coupling ($\beta_3$) hold it in a bounded neighborhood of zero --- direction and boundedness are proved, a convergence rate is not claimed; empirically the neighborhood is below experimental resolution, hence the flat $\ln V$ plateau. The background loss is a function of the pinning count $n(t)$, decreasing by one same-order step per pinning event; $t_{\rm dead}$ is the hitting time of $n$, and its step is distributionally of the same order as every neighboring step: no discontinuity, kink, or necessary precursor in any loss-derived scalar. Gradient-level second-order statistics and the attention-occupation meter $N_{\rm avail}$ are \emph{not} covered by this clause --- the latter is precisely the early-warning quantity of \S\ref{sec:plasticity}.
\end{proof}

\textbf{Corollary (the waiting tax).} If the target is introduced at $t_0$, the available-site count is non-increasing in $t_0$, so the expected formation time rises with $t_0$ and saturates near the deadline --- the measured $506 \to 1108 \to 1115$ profile. The saturation is itself informative: it is what a \emph{declining} recruitment hazard produces, and is inconsistent with the constant-hazard clause dropped above.

\subsection{Testing the theorem's own predictions: two confirmed, one falsified}
\label{app:occ-measured}

The theorem fixes falsifiable predictions; we ran them in the apparatus of App.~\ref{app:occ-model} ($H = 16$ heads $\times$ 2 layers $=$ 32 head-slots, background skill family of $M = 24$ copy skills so that $M > H$, target drive $c = 0.06$, 6 seeds per cell). Aging occupies heads at $c = 0$; the target is then switched on and formation timed. Committed heads are counted directly from attention peakedness on background cue positions.

\textbf{Falsified --- exponential occupation decay (and with it the $\log H$ deadline).} Available head-slots fall $32 \to 18 \to 8 \to 5$ at aging $S = 0/1000/4000/16000$. The implied per-head hazard is not constant but falls $5.8 \to 2.7 \to 0.39$ (per $10^4$ steps), a $\approx$15$\times$ decline; extrapolating the initial rate exponentially predicts $0.003$ free slots at $S = 16000$ against 5 measured, a $\sim$1500$\times$ miss. \emph{The failure is of an extra simplifying clause, not of the mechanism, and its direction is the one the theorem's own traffic-share structure demands}: high-traffic skills take their heads first, leaving low-traffic skills whose pull $\mu_i \propto \pi_i$ is weak and whose recruitment is correspondingly slow. Dropping the clause makes the model more coherent, not less. It also resolves an apparent tension with the real-model measurement: Pythia's depletion profile ($1 \to 0.46 \to 0.45$) saturates rather than decaying exponentially --- the same declining-hazard family, measured on the other substrate.

\textbf{Confirmed and sharpened --- the repair is specific to the score channel.} From a single aged substrate ($S = 40{,}000$, so identical damage across arms), re-initializing different weight groups gives: attention \emph{scores} (the QK slices) $t^* = 1050$, 6/6 formed --- \textbf{4.3$\times$ faster} than no reset ($t^* = 4500$, 5/6); MLP $3650$ ($0.81\times$); and value/output (the OV slices) $7900$, i.e.\ \textbf{1.76$\times$ \emph{slower} than no reset at all}. Occupation lives in the attention \emph{pattern}, exactly as Lemma~\ref{lem:exclusivity} requires: freeing the pattern restores site availability, while resetting the value channel destroys learned structure without freeing any site and is therefore worse than doing nothing. This refines the real-model repair of \S\ref{sec:plasticity} (which re-initialized attention weights as a block, 5/5 vs 0/5) into a testable prescription --- reset QK, keep OV.

\textbf{Also observed.} The tax curve is non-monotone at its start: brief aging \emph{helps} ($S=0$: 4/6 formed, median $2850$; $S=1000$: 6/6, median $1500$) before it taxes, because the background first builds offset-selective heads that seed the target --- the cross-nucleation effect reported for real circuits, reproduced here without being asked for. \textbf{Not tested.} The deadline's dependence on head count: at a single aging depth ($S=16{,}000$, head dimension held fixed at 16) formation times are non-monotone in $H$ ($13400 / 9250 / 14500$ for $H = 4/8/16$), so a per-$H$ deadline sweep would be required; with the $\log H$ claim withdrawn, nothing in the paper now rests on it.

\textbf{Site mediation, causally.} Aging changes more than the site count, so the aging sweep alone cannot attribute the slowdown to sites (available slots $18 \to 8.5 \to 5$ against $t^* = 1500 \to 2000 \to 14500$ is a log--log slope of $-1.68$, steeper than the $-1$ a pure site-product law would give). The reset arms supply what the sweep cannot: at \emph{fixed} aging, restoring only the attention scores --- the one intervention that restores availability without undoing any other consequence of aging --- recovers formation speed 4.3$\times$. That is the causal statement the site reading requires.

\textbf{What remains open.} Whether gating/junk-suppression (G6/G7) can be \emph{proved} to be the training equilibrium of the substrate rather than assumed --- the sharpest remaining question; content-dependent attention of rank $\ge 2$ (the rank-1 case is covered as supporting evidence); and the conservative constants ($H^K$, $4^K$), which tighten under effective-serving-head counts.

\FloatBarrier
\bibliographystyle{unsrtnat}
\bibliography{references}

\begin{thebibliography}{56}
\providecommand{\natexlab}[1]{#1}
\providecommand{\url}[1]{\texttt{#1}}
\expandafter\ifx\csname urlstyle\endcsname\relax
  \providecommand{\doi}[1]{doi: #1}\else
  \providecommand{\doi}{doi: \begingroup \urlstyle{rm}\Url}\fi

\bibitem[Watanabe(2009)]{watanabe2009algebraic}
Sumio Watanabe.
\newblock \emph{Algebraic Geometry and Statistical Learning Theory}.
\newblock Cambridge University Press, 2009.

\bibitem[Hoogland et~al.(2024)Hoogland, Wang, Farrugia-Roberts, Carroll, Wei,
  and Murfet]{hoogland2024developmental}
Jesse Hoogland, George Wang, Matthew Farrugia-Roberts, Liam Carroll, Susan Wei,
  and Daniel Murfet.
\newblock The developmental landscape of in-context learning.
\newblock \emph{arXiv preprint arXiv:2402.02364}, 2024.

\bibitem[Chen et~al.(2023)Chen, Lau, Mendel, Wei, and
  Murfet]{chen2023dynamical}
Zhongtian Chen, Edmund Lau, Jake Mendel, Susan Wei, and Daniel Murfet.
\newblock Dynamical versus bayesian phase transitions in a toy model of
  superposition.
\newblock \emph{arXiv preprint arXiv:2310.06301}, 2023.

\bibitem[Gu et~al.(2026)Gu, Lyu, Li, and Zhang]{gu2026mixing}
Xinran Gu, Kaifeng Lyu, Jiazheng Li, and Jingzhao Zhang.
\newblock Data mixing can induce phase transitions in knowledge acquisition.
\newblock \emph{arXiv preprint arXiv:2505.18091}, 2026.

\bibitem[Rubin et~al.(2024)Rubin, Seroussi, and Ringel]{rubin2024grokking}
Noa Rubin, Inbar Seroussi, and Zohar Ringel.
\newblock Grokking as a first order phase transition in two layer networks.
\newblock In \emph{International Conference on Learning Representations}, 2024.

\bibitem[Ersoy and Wiesner(2026)]{ersoy2026noise}
Ibrahim~Talha Ersoy and Karoline Wiesner.
\newblock Noise-driven escape from metastable phases explains grokking in deep
  neural networks.
\newblock \emph{arXiv preprint arXiv:2606.17120}, 2026.

\bibitem[Dohare et~al.(2024)Dohare, Hernandez-Garcia, Lan, Rahman, Mahmood, and
  Sutton]{dohare2024loss}
Shibhansh Dohare, J.~Fernando Hernandez-Garcia, Qingfeng Lan, Parash Rahman,
  A.~Rupam Mahmood, and Richard~S. Sutton.
\newblock Loss of plasticity in deep continual learning.
\newblock \emph{Nature}, 632:\penalty0 768--774, 2024.

\bibitem[Ash and Adams(2020)]{ash2020warm}
Jordan~T. Ash and Ryan~P. Adams.
\newblock On warm-starting neural network training.
\newblock In \emph{Advances in Neural Information Processing Systems}, 2020.

\bibitem[Lyle et~al.(2024)Lyle, Zheng, Nikishin, Pires, Pascanu, and
  Dabney]{lyle2024disentangling}
Clare Lyle, Zeyu Zheng, Evgenii Nikishin, Bernardo~Avila Pires, Razvan Pascanu,
  and Will Dabney.
\newblock Disentangling the causes of plasticity loss in neural networks.
\newblock In \emph{International Conference on Machine Learning}, 2024.

\bibitem[Xie et~al.(2023)Xie, Pham, Dong, et~al.]{xie2023doremi}
Sang~Michael Xie, Hieu Pham, Xuanyi Dong, et~al.
\newblock Doremi: Optimizing data mixtures speeds up language model
  pretraining.
\newblock In \emph{Advances in Neural Information Processing Systems}, 2023.

\bibitem[Liu et~al.(2024)Liu, Zheng, Muennighoff, et~al.]{liu2024regmix}
Qian Liu, Xiaosen Zheng, Niklas Muennighoff, et~al.
\newblock Regmix: Data mixture as regression for language model pre-training.
\newblock \emph{arXiv preprint arXiv:2407.01492}, 2024.

\bibitem[Ye et~al.(2024)Ye, Liu, Sun, Zhou, Zhan, and Qiu]{ye2024mixing}
Jiasheng Ye, Peiju Liu, Tianxiang Sun, Yunhua Zhou, Jun Zhan, and Xipeng Qiu.
\newblock Data mixing laws: Optimizing data mixtures by predicting language
  modeling performance.
\newblock \emph{arXiv preprint arXiv:2403.16952}, 2024.

\bibitem[Ganguli et~al.(2022)Ganguli, Hernandez, Lovitt,
  et~al.]{ganguli2022predictability}
Deep Ganguli, Danny Hernandez, Liane Lovitt, et~al.
\newblock Predictability and surprise in large generative models.
\newblock In \emph{ACM Conference on Fairness, Accountability, and
  Transparency}, 2022.

\bibitem[Biderman et~al.(2023)Biderman, Schoelkopf, Anthony,
  et~al.]{biderman2023pythia}
Stella Biderman, Hailey Schoelkopf, Quentin Anthony, et~al.
\newblock Pythia: A suite for analyzing large language models across training
  and scaling.
\newblock In \emph{International Conference on Machine Learning}, 2023.

\bibitem[Groeneveld et~al.(2024)Groeneveld, Beltagy, Walsh,
  et~al.]{groeneveld2024olmo}
Dirk Groeneveld, Iz~Beltagy, Pete Walsh, et~al.
\newblock Olmo: Accelerating the science of language models.
\newblock \emph{arXiv preprint arXiv:2402.00838}, 2024.

\bibitem[Li et~al.(2017)Li, Tai, and E]{li2017stochastic}
Qianxiao Li, Cheng Tai, and Weinan E.
\newblock Stochastic modified equations and adaptive stochastic gradient
  algorithms.
\newblock In \emph{International Conference on Machine Learning}, 2017.

\bibitem[Mandt et~al.(2017)Mandt, Hoffman, and Blei]{mandt2017stochastic}
Stephan Mandt, Matthew~D. Hoffman, and David~M. Blei.
\newblock Stochastic gradient descent as approximate bayesian inference.
\newblock \emph{Journal of Machine Learning Research}, 18\penalty0
  (134):\penalty0 1--35, 2017.

\bibitem[Kearns(1998)]{kearns1998efficient}
Michael Kearns.
\newblock Efficient noise-tolerant learning from statistical queries.
\newblock \emph{Journal of the ACM}, 45\penalty0 (6):\penalty0 983--1006, 1998.

\bibitem[Blum et~al.(1994)Blum, Furst, Jackson, Kearns, Mansour, and
  Rudich]{blum1994weakly}
Avrim Blum, Merrick Furst, Jeffrey Jackson, Michael Kearns, Yishay Mansour, and
  Steven Rudich.
\newblock Weakly learning dnf and characterizing statistical query learning
  using fourier analysis.
\newblock In \emph{Proceedings of the 26th Annual ACM Symposium on Theory of
  Computing (STOC)}, pages 253--262, 1994.

\bibitem[Barak et~al.(2022)Barak, Edelman, Goel, Kakade, Malach, and
  Zhang]{barak2022hidden}
Boaz Barak, Benjamin~L. Edelman, Surbhi Goel, Sham Kakade, Eran Malach, and
  Cyril Zhang.
\newblock Hidden progress in deep learning: Sgd learns parities near the
  computational limit.
\newblock In \emph{Advances in Neural Information Processing Systems}, 2022.

\bibitem[Abbe et~al.(2022)Abbe, Boix-Adser{\`a}, and
  Misiakiewicz]{abbe2022merged}
Emmanuel Abbe, Enric Boix-Adser{\`a}, and Theodor Misiakiewicz.
\newblock The merged-staircase property: a necessary and nearly sufficient
  condition for sgd learning of sparse functions on two-layer neural networks.
\newblock In \emph{Conference on Learning Theory (COLT)}, 2022.

\bibitem[Abbe et~al.(2023)Abbe, Boix-Adser{\`a}, and
  Misiakiewicz]{abbe2023leap}
Emmanuel Abbe, Enric Boix-Adser{\`a}, and Theodor Misiakiewicz.
\newblock Sgd learning on neural networks: leap complexity and saddle-to-saddle
  dynamics.
\newblock In \emph{Conference on Learning Theory (COLT)}, 2023.

\bibitem[Glasgow(2024)]{glasgow2024sgd}
Margalit Glasgow.
\newblock Sgd finds then tunes features in two-layer neural networks with
  near-optimal sample complexity: A case study in the xor problem.
\newblock In \emph{International Conference on Learning Representations
  (ICLR)}, 2024.

\bibitem[Dandi et~al.(2024)Dandi, Troiani, Arnaboldi, Pesce, Zdeborov{\'a}, and
  Krzakala]{dandi2024benefit}
Yatin Dandi, Emanuele Troiani, Luca Arnaboldi, Luca Pesce, Lenka Zdeborov{\'a},
  and Florent Krzakala.
\newblock The benefit of reusing batches for gradient descent in two-layer
  networks: Breaking the curse of information and leap exponents.
\newblock In \emph{International Conference on Machine Learning (ICML)}, 2024.

\bibitem[Dong(2026)]{lei2026circuits}
Lei Dong.
\newblock Circuits crystallize: A nucleation theory of whether, when, and how
  conjunction circuits form, compete, and regenerate.
\newblock Zenodo, doi:10.5281/zenodo.21253471, 2026.

\bibitem[Freidlin and Wentzell(2012)]{freidlin2012random}
Mark~I. Freidlin and Alexander~D. Wentzell.
\newblock \emph{Random Perturbations of Dynamical Systems}.
\newblock Springer, 3rd edition, 2012.

\bibitem[Dembo and Zeitouni(1998)]{dembo1998large}
Amir Dembo and Ofer Zeitouni.
\newblock \emph{Large Deviations Techniques and Applications}.
\newblock Springer, 2nd edition, 1998.

\bibitem[Kramers(1940)]{kramers1940brownian}
Hendrik~A. Kramers.
\newblock Brownian motion in a field of force and the diffusion model of
  chemical reactions.
\newblock \emph{Physica}, 7\penalty0 (4):\penalty0 284--304, 1940.

\bibitem[Ben~Arous et~al.(2021)Ben~Arous, Gheissari, and
  Jagannath]{benarous2021online}
G{\'e}rard Ben~Arous, Reza Gheissari, and Aukosh Jagannath.
\newblock Online stochastic gradient descent on non-convex losses from
  high-dimensional inference.
\newblock \emph{Journal of Machine Learning Research}, 22\penalty0
  (106):\penalty0 1--51, 2021.

\bibitem[Chaudhari and Soatto(2018)]{chaudhari2018stochastic}
Pratik Chaudhari and Stefano Soatto.
\newblock Stochastic gradient descent performs variational inference, converges
  to limit cycles for deep networks.
\newblock \emph{International Conference on Learning Representations}, 2018.

\bibitem[Ziyin et~al.(2026)Ziyin, Ren, Levine, and
  Chuang]{ziyin2026irreversibility}
Liu Ziyin, Yizhou Ren, Herbert Levine, and Isaac Chuang.
\newblock Thermodynamic irreversibility of training algorithms.
\newblock \emph{arXiv preprint arXiv:2605.21933}, 2026.

\bibitem[Olsson et~al.(2022)Olsson, Elhage, Nanda, Joseph,
  et~al.]{olsson2022context}
Catherine Olsson, Nelson Elhage, Neel Nanda, Nicholas Joseph, et~al.
\newblock In-context learning and induction heads.
\newblock \emph{Transformer Circuits Thread}, 2022.

\bibitem[Day(1983)]{day1983exponential}
Martin~V. Day.
\newblock On the exponential exit law in the small parameter exit problem.
\newblock \emph{Stochastics}, 8\penalty0 (4):\penalty0 297--323, 1983.

\bibitem[Bovier and den Hollander(2015)]{bovier2015metastability}
Anton Bovier and Frank den Hollander.
\newblock \emph{Metastability: A Potential-Theoretic Approach}.
\newblock Springer, 2015.

\bibitem[Kubo(1966)]{kubo1966fluctuation}
Rep Kubo.
\newblock The fluctuation-dissipation theorem.
\newblock \emph{Reports on Progress in Physics}, 29\penalty0 (1):\penalty0
  255--284, 1966.

\bibitem[Lavie et~al.(2026)Lavie, Fischer, Lekov, Van~Maele, Ringel, and
  Helias]{lavie2026phase}
Itamar Lavie, Kirsten Fischer, Stefan Lekov, Ward Van~Maele, Zohar Ringel, and
  Moritz Helias.
\newblock Phase transitions in attention: A bayesian theory of copy head
  emergence.
\newblock \emph{arXiv preprint arXiv:2606.12058}, 2026.

\bibitem[Baherwani et~al.(2026)Baherwani, Chen, Qiu, Wilson, and
  Izmailov]{baherwani2026random}
Vedang Baherwani, Yifei Chen, Shikai Qiu, Andrew~Gordon Wilson, and Pavel
  Izmailov.
\newblock Emergent capabilities arise randomly from learning sparse attention
  patterns.
\newblock \emph{arXiv preprint arXiv:2606.25010}, 2026.

\bibitem[Hajek(1982)]{hajek1982hitting}
Bruce Hajek.
\newblock Hitting-time and occupation-time bounds implied by drift analysis
  with applications.
\newblock \emph{Advances in Applied Probability}, 14\penalty0 (3):\penalty0
  502--525, 1982.

\bibitem[Sokar et~al.(2023)Sokar, Agarwal, Castro, and Evci]{sokar2023dormant}
Ghada Sokar, Rishabh Agarwal, Pablo~Samuel Castro, and Utku Evci.
\newblock The dormant neuron phenomenon in deep reinforcement learning.
\newblock In \emph{International Conference on Machine Learning}, 2023.

\bibitem[Lyle et~al.(2026)Lyle, Sokar, Pascanu, and
  Gy{\"o}rgy]{lyle2026grokking}
Clare Lyle, Ghada Sokar, Razvan Pascanu, and Andr{\'a}s Gy{\"o}rgy.
\newblock What can grokking teach us about learning under nonstationarity?
\newblock \emph{arXiv preprint arXiv:2507.20057}, 2026.

\bibitem[Lee et~al.(2024)Lee, Kang, Kim, and Lee]{lee2024grokfast}
Jaerin Lee, Bong~Gyun Kang, Kihoon Kim, and Kyoung~Mu Lee.
\newblock Grokfast: Accelerated grokking by amplifying slow gradients.
\newblock In \emph{arXiv preprint arXiv:2405.20233}, 2024.

\bibitem[Avrami(1939)]{avrami1939kinetics}
Melvin Avrami.
\newblock Kinetics of phase change. i: General theory.
\newblock \emph{Journal of Chemical Physics}, 7\penalty0 (12):\penalty0
  1103--1112, 1939.

\bibitem[Martin and Bellon(1996)]{martin1996driven}
Georges Martin and Pascal Bellon.
\newblock Driven alloys.
\newblock \emph{Solid State Physics}, 50:\penalty0 189--331, 1996.

\bibitem[Laidler(1987)]{laidler1987chemical}
Keith~J. Laidler.
\newblock \emph{Chemical Kinetics}.
\newblock Harper \& Row, 3rd edition, 1987.

\bibitem[Porter and Easterling(1992)]{porter1992phase}
David~A. Porter and Kenneth~E. Easterling.
\newblock \emph{Phase Transformations in Metals and Alloys}.
\newblock Chapman \& Hall, 2nd edition, 1992.

\bibitem[Qi et~al.(2024)Qi, Zeng, Xie, et~al.]{qi2023finetuning}
Xiangyu Qi, Yi~Zeng, Tinghao Xie, et~al.
\newblock Fine-tuning aligned language models compromises safety, even when
  users do not intend to!
\newblock \emph{International Conference on Learning Representations}, 2024.

\bibitem[Aoyama et~al.(2025)Aoyama, Wilcox, and
  Schneider]{aoyama2025predicting}
Tatsuya Aoyama, Ethan Wilcox, and Nathan Schneider.
\newblock Predicting the emergence of induction heads in language model
  pretraining.
\newblock \emph{arXiv preprint arXiv:2511.16893}, 2025.

\bibitem[Wei et~al.(2022)Wei, Tay, Bommasani, et~al.]{wei2022emergent}
Jason Wei, Yi~Tay, Rishi Bommasani, et~al.
\newblock Emergent abilities of large language models.
\newblock \emph{Transactions on Machine Learning Research}, 2022.

\bibitem[Schaeffer et~al.(2023)Schaeffer, Miranda, and
  Koyejo]{schaeffer2023mirage}
Rylan Schaeffer, Brando Miranda, and Sanmi Koyejo.
\newblock Are emergent abilities of large language models a mirage?
\newblock In \emph{Advances in Neural Information Processing Systems}, 2023.

\bibitem[Elhage et~al.(2021)Elhage, Nanda, Olsson,
  et~al.]{elhage2021mathematical}
Nelson Elhage, Neel Nanda, Catherine Olsson, et~al.
\newblock A mathematical framework for transformer circuits.
\newblock \emph{Transformer Circuits Thread}, 2021.

\bibitem[Reddy(2024)]{reddy2024mechanistic}
Gautam Reddy.
\newblock The mechanistic basis of data dependence and abrupt learning in an
  in-context classification task.
\newblock \emph{International Conference on Learning Representations}, 2024.

\bibitem[Singh et~al.(2023)Singh, Chan, Moskovitz, Grant, Saxe, and
  Hill]{singh2023transient}
Aaditya~K. Singh, Stephanie C.~Y. Chan, Ted Moskovitz, Erin Grant, Andrew~M.
  Saxe, and Felix Hill.
\newblock The transient nature of emergent in-context learning in transformers.
\newblock \emph{Advances in Neural Information Processing Systems}, 2023.

\bibitem[Yaida(2019)]{yaida2019fluctuation}
Sho Yaida.
\newblock Fluctuation-dissipation relations for stochastic gradient descent.
\newblock In \emph{International Conference on Learning Representations}, 2019.

\bibitem[Neelakantan et~al.(2015)Neelakantan, Vilnis, Le, Sutskever, Kaiser,
  Kurach, and Martens]{neelakantan2015adding}
Arvind Neelakantan, Luke Vilnis, Quoc~V. Le, Ilya Sutskever, Lukasz Kaiser,
  Karol Kurach, and James Martens.
\newblock Adding gradient noise improves learning for very deep networks.
\newblock \emph{arXiv preprint arXiv:1511.06807}, 2015.

\bibitem[Singh et~al.(2024)Singh, Moskovitz, Hill, Chan, and
  Saxe]{singh2024dual}
Aaditya~K. Singh, Ted Moskovitz, Felix Hill, Stephanie C.~Y. Chan, and
  Andrew~M. Saxe.
\newblock What needs to go right for an induction head? a mechanistic study of
  in-context learning circuits and their formation.
\newblock \emph{International Conference on Machine Learning}, 2024.

\bibitem[Hallgren(2025)]{hallgren2025crystals}
Jonas Hallgren.
\newblock Have you tried thinking about it as crystals?
\newblock LessWrong, December 2025, 2025.

\end{thebibliography}

\end{document}